\documentclass{article} 
\usepackage{iclr2027_conference,times}

\usepackage{amsmath,amsfonts,bm}

\def\eqref#1{equation~\ref{#1}}

\def\1{\bm{1}}

\DeclareMathAlphabet{\mathsfit}{\encodingdefault}{\sfdefault}{m}{sl}
\SetMathAlphabet{\mathsfit}{bold}{\encodingdefault}{\sfdefault}{bx}{n}

\DeclareMathOperator*{\argmax}{arg\,max}

\usepackage{hyperref}
\usepackage{url}
\usepackage{caption}

\usepackage{amsmath,amssymb,amsthm}
\usepackage{algorithm}
\usepackage{algpseudocode}
\usepackage{booktabs, tabularx, ragged2e, multirow, amsmath}
\usepackage{booktabs}
\usepackage{longtable}

\usepackage{booktabs}
\usepackage{multirow}
\usepackage{makecell}

\usepackage{graphicx}

\usepackage{wrapfig}
\usepackage{graphicx}

\usepackage{xcolor}
\usepackage{hyperref}

\definecolor{navyblue}{RGB}{31, 119, 180}

\hypersetup{
    colorlinks=true,       
    citecolor=navyblue,    
    linkcolor=navyblue,    
    urlcolor=navyblue      
}

\usepackage{array}
\newcolumntype{Y}{>{\centering\arraybackslash}X}

\usepackage[normalem]{ulem}
\newlength{\fmcol}
 
\providecommand{\argmax}{\operatorname*{arg\,max}}

\newcommand{\ms}[2]{$#1_{\,#2}$}                        
\newcommand{\msb}[2]{$\underline{\mathbf{#1}}_{\,#2}$}  

\newtheorem{lemma}{Lemma}

\title{Label-Free Coreset Selection with Foundation Models for Efficient Annotation in Computational Pathology}

\author{%
Tuo Yin$^{1,2}$,\;
Jennifer Dhont$^{1,2,3}$ \\[4pt]
$^{1}$Department of Medical Physics, Institut Jules Bordet, H\^{o}pital Universitaire de Bruxelles (HUB), Belgium \\
$^{2}$Radiophysics and MRI Physics Laboratory, Universit\'{e} Libre de Bruxelles (ULB), Belgium \\
$^{3}$Laboratory of Image Synthesis and Analysis (LISA), Universit\'{e} Libre de Bruxelles (ULB), Belgium \\[4pt]
\texttt{\{tuo.yin, jennifer.dhont\}@hubruxelles.be}
}

\iclrfinalcopy 
\begin{document}

\maketitle
\thispagestyle{plain}
\begin{abstract}
Computational pathology has the potential to improve clinical outcomes through a demonstrated increase in diagnostic and prognostic accuracy. However, the development and validation of deep learning algorithms still require annotated data, a costly procedure involving expert pathologists who already face critical workforce shortages. Existing coreset selection methods to optimize annotation efforts currently all rely on hyperparameters tuned on natural-image benchmarks that do not transfer to histopathology and are cumbersome to use in clinical practice. In this study, we present GCcore, a novel label-free coreset selection method that embeds every image of a dataset with any pathology foundation model and greedily selects the samples that collectively maximize the global coverage of the embedding space. The proposed method provides a lower-bound guarantee on the global coverage of the returned coreset for any coreset size, while being completely hyperparameter-free and deterministic. We demonstrate GCcore's superior performance over 14 baselines including state-of-the-art methods across 10 tasks and datasets spanning whole slide image classification, tile classification, and tissue segmentation, where it ranks first on six and within the top three on nine, while also demonstrating how existing methods can shift by up to five rank positions depending on their hyperparameter settings. \href{https://github.com/OncoAI-ULBHUB/GCcore}{Code} is publicly available.
\end{abstract}

\section{Introduction}
\label{sec:intro}

\newcommand{\rankfigfrac}{0.45}
\setlength{\intextsep}{0pt}        
\setlength{\belowcaptionskip}{0pt} 
\begin{wrapfigure}{b}{\rankfigfrac\linewidth}
  \centering
  \includegraphics[width=\linewidth]{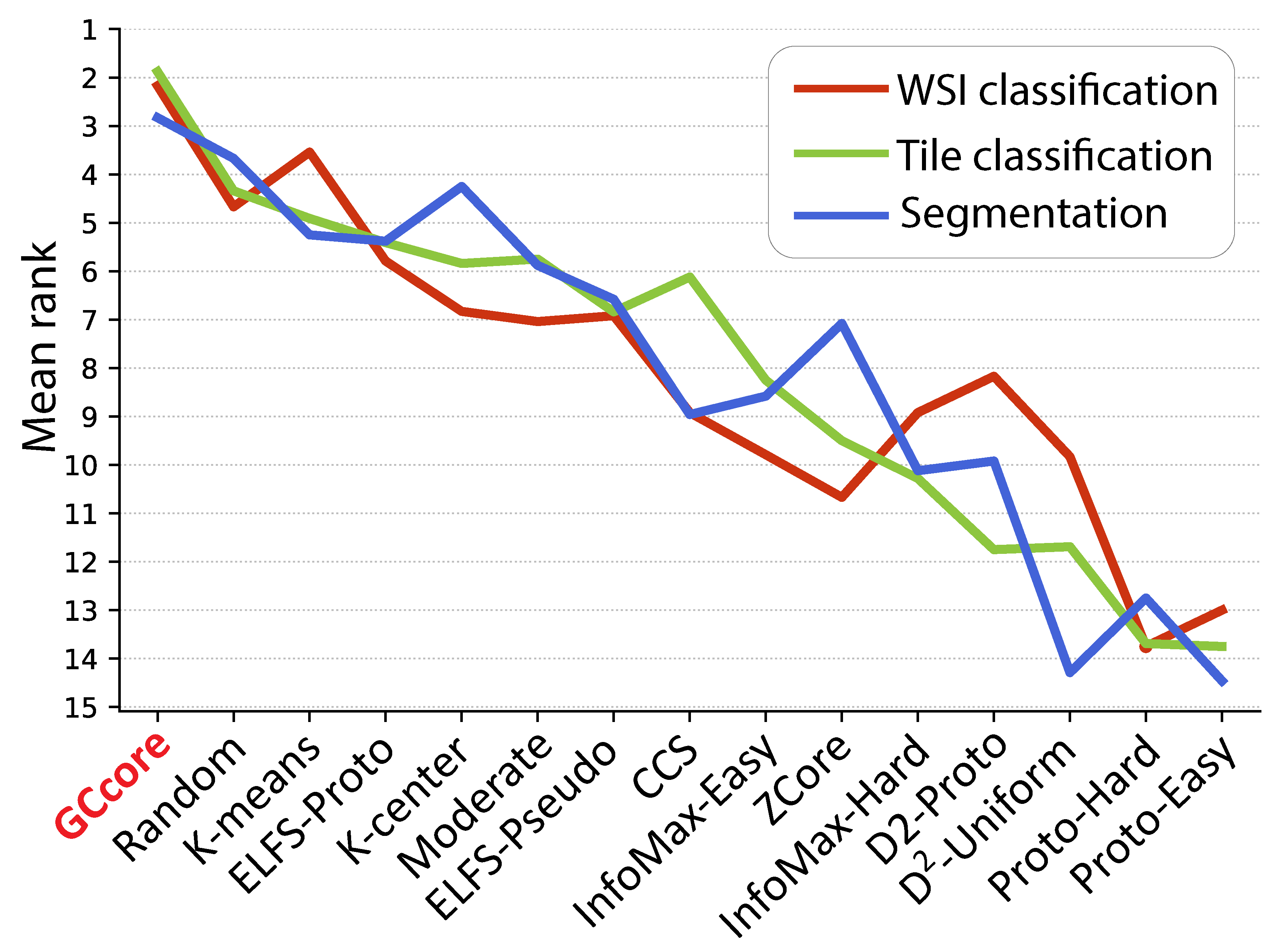}
  \caption{Mean rank of GCcore and competing label-free selection methods across 10 datasets spanning three common computational pathology tasks, in the low-retention regime (tens to hundreds of samples in the coreset that require annotation). Lower rank indicates better downstream performance.}
  \label{fig:three_task_mean_rank}
\end{wrapfigure}

Histopathology plays a central role in clinical practice, underpinning a wide range of diagnostic, , prognostic, and treatment selection decisions, particularly in oncology. At the same time, recent advances in artificial intelligence (AI) have demonstrated considerable potential to enhance and extend histopathological analysis. Digitization of whole-slide images (WSIs) and the increasing use of computational pathology has enabled faster and more standardized (semi)automated tissue evaluation~\citep{verghese2023computational}. Deep learning (DL)-based methods, in particular, have achieved performance comparable to that of expert pathologists across a number of tasks, while also broadening the scope to tasks that were previously too time-consuming, such as detailed WSI tissue segmentation. More recently, large-scale pathology foundation models (FMs) pretrained on millions of image tiles have demonstrated to be robust feature extractors~\citep{campanella2025clinical,bareja2025evaluating}, enabling a wide range of downstream tasks while requiring smaller amounts of annotated data~\citep{roth2024low,nicke2025tissue}.

Nevertheless, the development of DL models in computational pathology still relies heavily on expert-annotated data, whether for fully supervised training or for fine-tuning in the context of FMs. Obtaining such high-quality annotations is time-consuming and cognitively demanding~\citep{montezuma2023}, motivating a fundamental question: given an unlabeled histopathology dataset of $N$ samples and a budget of $k \ll N$ annotations, which $k$ WSIs or tiles, i.e., which coreset, should be selected for expert annotation to maximize downstream performance? Selecting a coreset without access to prior label information remains challenging, and, to date, no systematic and reproducible strategy has been established and validated for histopathology.

In computer vision, coreset selection methods have been proposed, but with the primary goal of reducing computational cost during training rather than optimizing annotation efforts~\citep{coleman2019selection,zheng2025elfs,griffin2026zero}. Consequently, most methods were formulated with access to labels and only subsequently adapted to the label-free regime. Existing approaches can be broadly grouped into three categories. \emph{Coverage-based} methods select samples according to their distribution in the embedding space, aiming to minimize redundancy among selected samples. \emph{Importance-based} methods assign each sample a scalar importance score, derived from training dynamics or embedding geometry, and select samples based on this score alone. \emph{Hybrid} methods combine these two principles by jointly accounting for the individual importance of each sample and the redundancy among selected samples.

However, we show that existing coreset selection methods cannot be reliably transferred to histopathology datasets for two main reasons. First, existing label-free coreset selection methods usually rely on hyperparameters that were selected based on downstream performance on a limited set of natural-image benchmarks (a full list is given in Appendix~\ref{Existing_methods}), and neither the values nor the ranges of robustness established in this context are necessarily transferable. For instance, InfoMax~\citep{tan2025data} was reported to be robust to its pairwise weight $\alpha$ over the range $[0.3, 4]$ on CC12M. On PANDA at a 2\% retention rate, however, varying $\alpha$ within this range yields quadratic weighted Cohen's $\kappa$ (QWK) values between 0.58 and 0.70.
Similarly, ELFS~\citep{zheng2025elfs} sets the number of pseudo-classes $\tilde{C}$ to the number of classes in the dataset. On TCGA-BRCA (binary classification) at a 5\% retention rate, varying $\tilde{C}$ from 2 to 20 yields test balanced accuracies between 0.71 and 0.85, with the prescribed value $\tilde{C} = 2$ performing worst. Moreover, this rule does not transfer to the four-class segmentation task of BEETLE, where individual tiles do not carry a single class label; varying $\tilde{C}$ from 4 to 25 at a 0.02\% retention rate yields test weighted Dice values between 0.70 and 0.84.
These examples show that (1)~a range reported as robust on natural images can produce substantial variation on histopathology data; (2)~the recommended value can be among the worst choices; and (3)~the default may be undefined, as for segmentation. Moreover, it is also unclear whether hyperparameters optimized for histopathology generalize across different cancer types and tasks, or whether they need to be determined separately for each new application, which may itself require additional labelled data in contradiction with the main goal of coreset selection. This motivates a genuinely hyperparameter-free coreset selection method for computational pathology.
Second, existing methods generally provide no \emph{a priori} guarantee on the quality of the selected subset with respect to a specified objective function. Such a guarantee is of particular value in histopathology, where every selected sample is converted into a dense annotation at a cost measured in pathologist hours, and where this cost cannot be recovered once a selection has proven poor.

To address these two challenges, we propose GCcore, a novel label-free coreset selection method that iteratively adds the sample from the remaining pool that maximizes the global coverage (GC) of the embedding space (Algorithm~\ref{alg:coreset_selection}). It accounts for both coverage and importance without a weighting term, unlike existing hybrid methods: a sample is selected only when it is both representative, i.e., surrounded by many samples, and novel, i.e., far from the samples already selected. We compare GCcore with 14 baselines on 10 pathology datasets spanning WSI classification, tile classification and segmentation (Section~\ref{exp}). GCcore achieves the best mean rank on all three task families (2.17, 1.88 and 2.75, respectively; Fig.~\ref{fig:three_task_mean_rank}), ranking first on six datasets and within the top three on nine, while remaining computationally efficient: it selects a coreset of 100 samples from a pool of 100{,}000 in 89~s. We further prove GCcore attains a lower bound guarantee $1-1/e$ on the GC objective on any dataset and at every coreset size (Section~\ref{methods}). 

Our contributions are:
(1)~\textbf{A new selection principle.} We formulate label-free coreset selection as the maximization of a scale-invariant GC objective, defined as the area under the coverage--radius curve, and show that its greedy maximization balances representativeness and diversity without any weighting term. The resulting method, GCcore, is hyperparameter-free and deterministic, applies to the embeddings of any FM, and requires no labels.
(2)~\textbf{An \emph{a priori} guarantee at every budget.}  We prove that GC is monotone and submodular, so that the greedy solution attains at least $1-1/e$ of the maximum GC achievable at every coreset size. Its practical value is that pathologists can annotate samples in the order chosen and stop whenever the budget is exhausted, with a certified bound on the coverage of the annotated subset, whereas existing label-free methods generally offer no such assurance.
(3)~\textbf{Comprehensive evaluation and superior performance.} We benchmark 15 label-free coreset selection methods on 10 histopathology datasets and three tasks, including segmentation, the most annotation-intensive task in histopathology, on which none of the existing methods had previously been evaluated. GCcore achieves the best overall mean rank (2.2), whereas no baseline is consistently competitive; moreover, the rank of a baseline shifts by 0.9 to 2.1 positions on average, and by up to 4.5 positions on a single dataset, depending solely on its own hyperparameters within recommended intervals (Table~\ref{tab:rank_sd}).

\section{Related work}
We discuss below the existing coreset selection methods that are label-free by design or that can be adapted into a label-free variant (see Appendix~\ref{Existing_methods} for a full list), including coverage-based, importance-based, and hybrid methods.

\paragraph{Coverage-based selection.}
Coverage-based methods select samples according to their arrangement in the embedding space and differ in the covering criterion they optimize. Clustering-based selection partitions the embedding space into groups and selects representatives from each~\citep{macqueen1967multivariate,ester1996density,arthur2007k}; how samples are allocated across clusters (uniformly or proportionally to cluster size) and which representatives are chosen within each cluster (the sample nearest the centroid, or samples stratified by distance from it) are hyperparameters with no principled label-free setting. K-center greedy~\citep{sener2017active} minimizes the radius needed to cover the single least-covered point; the selection is therefore dictated by the extremes of the distribution and systematically favors outliers over representative samples.

\paragraph{Importance-based selection.}
Importance-based methods assign a scalar score to each sample and select samples based on this score alone, thereby treating set-level information as a sum of independent contributions. The earliest methods derive scores from training dynamics~\citep{toneva2018empirical, pleiss2020identifying, paul2021deep}. All require labels and training on the full dataset, and are thus inapplicable to annotation planning.
Prototype~\citep{sorscher2022beyond} clusters features with K-means and scores each sample by its distance to its cluster centroid, keeping atypical samples when the budget is abundant and prototypical samples when it is scarce. 
Moderate coreset~\citep{xia2022moderate} instead keeps samples whose distance to their class center is closest to the class median. Without labels, both require the number of clusters, and Prototype additionally requires the budget threshold at which the pruning mode switches. 
CCS~\citep{zheng2022coverage} prunes a budget-dependent fraction of the hardest samples by AUM and samples uniformly from equal-width bins of the remaining score range; the pruning rate and the number of bins must be set for each budget.  
Finally, ELFS~\citep{zheng2025elfs} recovers training-dynamics scores through pseudo-labels obtained by deep clustering and prunes both the hardest and the easiest samples. Its selection inherits the errors of the clustering step, depends on the proxy model's training settings, and requires the number of pseudo-classes and the pruning fraction.

\paragraph{Hybrid methods.}
Hybrid methods account for both the importance of individual samples and the redundancy among selected samples, but predominantly through heuristic procedures that require weighting the two. D\textsuperscript{2}-Pruning~\citep{maharana2024mathbb} builds a $k$-nearest-neighbor graph over difficulty-weighted samples, propagates neighborhood difficulty into each node's score by forward message passing, and then repeatedly selects the highest-scoring node while discounting its neighbors by reverse message passing. The number of neighbors and the forward and reverse message-passing weights, which balance difficulty against redundancy, must be fixed in advance. InfoMax~\citep{tan2025data} defines an explicit set-level objective, the sum of per-sample importance scores minus an $\alpha$-weighted sum of pairwise similarities among selected samples. In its unsupervised mode, the per-sample term is the Prototype score~\citep{sorscher2022beyond}, so InfoMax inherits the hyperparameters of that score in addition to $\alpha$. ZCore~\citep{griffin2026zero} draws random probes in random low-dimensional subspaces, rewards the sample nearest each probe while penalizing its neighbors, and retains the top-$k$ scorers; the procedure is stochastic and trades runtime against performance through the number of sampling iterations, which must be set together with the dimensionality of the subspaces.

In summary, every method above exposes at least one hyperparameter whose value was calibrated on natural-image benchmarks and cannot be re-tuned on a new cohort without the labels that coreset selection is meant to spare. We treat the random seed as a hyperparameter as well, since for stochastic methods, such as K-center, K-means and ZCore, it determines the selected subset. Moreover, none of these methods offers an \emph{a priori} guarantee for an objective that reflects coverage of the data distribution as a whole. A third limitation concerns evaluation: none of these methods have been assessed on segmentation, the most annotation-intensive task in histopathology, and the one for which the number of clusters required by clustering-based selection, Prototype and ELFS is hardest to set, since it does not correspond to any observable notion of class.

\section{Methods}
\label{methods}

As illustrated in Fig.~\ref{fig:overview}(a), the overall pipeline consists of three steps. First, the entire unlabeled dataset $\mathcal{S}=\{x_1,x_2,\dots,x_N\}$ is encoded by a pre-trained FM $f(\cdot)$ into embeddings $\mathbf{E}=\{\mathbf{e}_1,\mathbf{e}_2,\dots,\mathbf{e}_N\}$, where $\mathbf{e}_i=f(x_i)\in\mathbb{R}^{D}$, $N$ is the number of samples, and $D$ is the feature dimension. Second, our proposed method, the global coverage-based coreset selection algorithm (GCcore), identifies a representative subset $\mathcal{S}'\subset \mathcal{S}$ of size $\lvert \mathcal{S}'\rvert=k\ll N$ that covers the data distribution in the embedding space through iterative global coverage maximization; the selection mechanism is illustrated in Fig.~\ref{fig:overview}(b) and detailed below. Finally, the selected coreset is annotated by pathologists to provide ground-truth labels, which serve as the basis for supervised training. Table~\ref{tab:notation} in Appendix~\ref{notation} summarizes the notation used throughout.


\begin{figure}[t]
  \centering
  \includegraphics[width=0.85\linewidth]{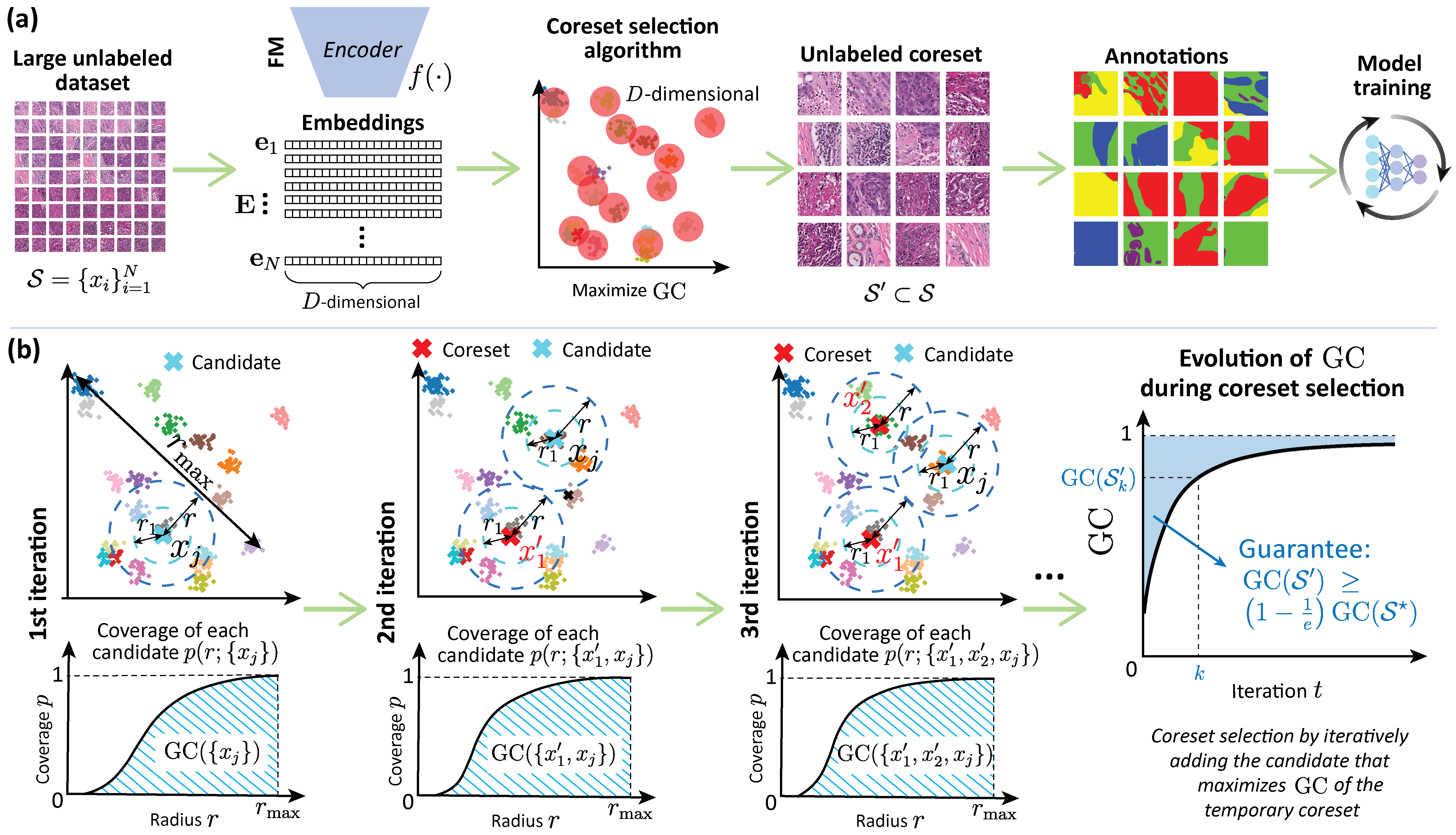}
  \caption{(a) Overview of the label-free coreset selection methodology. Embeddings of the full training set are first extracted using any histopathology FM and represented in feature space, after which the proposed algorithm identifies a subset of samples that collectively maximize coverage, defining the coreset. Annotations for these selected samples are then obtained from pathologists and used to train any task-specific model. (b) Schematic illustration of the proposed coreset selection algorithm. Each iteration adds the single sample whose marginal gain in $\mathrm{GC}$, evaluated jointly with the samples already selected, is largest, so that $\mathrm{GC}$ increases monotonically as the coreset grows.}
  \label{fig:overview}
\end{figure}


\paragraph{Coverage measure.}
A natural way to quantify how well a subset $\mathcal{T}\subseteq \mathcal{S}$ covers the dataset is the fraction of samples that lie within a radius $r$ of their nearest subset member. For each sample $x_i$, let
\begin{equation}
    d_i(\mathcal{T}) \;=\; \min_{x_j\in \mathcal{T}}\lVert \mathbf{e}_i-\mathbf{e}_j\rVert_1 , \qquad \forall\, x_i \in \mathcal{S} ,
    \label{eq:di}
\end{equation}
denote its nearest-neighbor distance to $\mathcal{T}$, and collect these into $\mathbf{d}(\mathcal{T})=\big(d_1(\mathcal{T}),\dots,d_N(\mathcal{T})\big)^{\!\top}\in\mathbb{R}^{N}$, where $\lVert\cdot\rVert_1$ is the $\ell_1$ (Manhattan) norm, $\lVert \mathbf{e}_i-\mathbf{e}_j\rVert_1=\sum_{m=1}^{D}\lvert e_{im}-e_{jm}\rvert$, which we adopt throughout unless stated otherwise. The coverage at radius $r$ is
\begin{equation}
    p(r;\mathcal{T}) \;=\; \frac{1}{N}\bigl\lvert\{\,x_i \in \mathcal{S} : d_i(\mathcal{T})\leq r\,\}\bigr\rvert , 
    \label{eq:pr}
\end{equation}
where $\lvert\cdot\rvert$ denotes the cardinality of a set. The scalar $p(r;\mathcal{T})$ is highly sensitive to the choice of $r$, which is arbitrary and varies with the scale of the embeddings across datasets. To obtain a scale-robust measure, we integrate $p(r;\mathcal{T})$ over all radii up to the maximum pairwise distance $r_{\max}=\max_{i<j}\lVert \mathbf{e}_i-\mathbf{e}_j\rVert_1$, a constant determined by $\mathcal{S}$ alone and strictly positive whenever $\mathcal{S}$ contains at least two distinct embeddings, and normalize, defining the global coverage as the area under the coverage--radius curve:
\begin{equation}
    \mathrm{GC}(\mathcal{T}) \;=\; \frac{1}{r_{\max}}\int_{0}^{r_{\max}} p(r;\mathcal{T})\,dr .
    \label{eq:gc}
\end{equation}
 
\paragraph{Global objective.}
Coreset selection is then formulated as the maximization of coverage subject to a cardinality budget,
\begin{equation}
    \mathcal{S}^{\star} \;=\; \argmax_{\mathcal{T}\subseteq \mathcal{S},\;\lvert \mathcal{T}\rvert=k}\; \mathrm{GC}(\mathcal{T}).
    \label{eq:objective}
\end{equation}
In contrast to per-sample importance-scoring methods, which assign each example a fixed importance score independently of which other examples are selected, Eq.~\ref{eq:objective} specifies a single global objective defined over the selected set as a whole, enabling the approximation guarantee established in Appendix~\ref{sec:guarantee}.
 
\paragraph{Efficient coverage evaluation.}
Directly estimating $\mathrm{GC}$ requires sampling
$r_{\text{steps}}$ radii to evaluate the coverage $p(r;\mathcal{T})$ and integrating the resulting curve, incurring an $O(N r_{\text{steps}})$ cost per candidate evaluation. We avoid the radius sweep entirely via the tail-integral identity for the nonnegative quantities $d_i(\mathcal{T})$, which satisfy $d_i(\mathcal{T})\le r_{\max}$ for every $x_i\in \mathcal{S}$:
\begin{equation}
    \bar{d}(\mathcal{T}) \;=\; \frac{1}{N}\sum_{i=1}^{N} d_i(\mathcal{T}) \;=\; \frac{1}{N}\sum_{i=1}^{N}\int_{0}^{r_{\max}} \mathbf{1}\{d_i(\mathcal{T})>r\}\,dr .
    \label{eq:tail}
\end{equation}
Substituting the definition of $p(r;\mathcal{T})$ from Eq.~\ref{eq:pr} gives
\begin{equation}
    \bar{d}(\mathcal{T}) \;=\; \int_{0}^{r_{\max}} \frac{1}{N}\sum_{i=1}^{N}\mathbf{1}\{d_i(\mathcal{T})>r\}\,dr
    \;=\; \int_{0}^{r_{\max}}\bigl(1-p(r;\mathcal{T})\bigr)\,dr ,
    \label{eq:dbar}
\end{equation}
and comparison with Eq.~\ref{eq:gc} yields
\begin{equation}
    \mathrm{GC}(\mathcal{T}) \;=\; 1-\frac{\bar{d}(\mathcal{T})}{r_{\max}} .
    \label{eq:equiv}
\end{equation}
The global objective of maximizing $\mathrm{GC}$ is therefore equivalent to minimizing the mean nearest-coreset distance $\bar{d}$. Beyond removing the $O(r_{\text{steps}})$ factor and its discretization error, we note that, when a sample $x_j$ is added to the coreset, the nearest-coreset distances need only be updated by an elementwise minimum,
\begin{equation}
    d_i(\mathcal{T}\cup\{x_j\}) \;=\; \min\!\big( d_i(\mathcal{T}),\, \lVert \mathbf{e}_{i} - \mathbf{e}_{j} \rVert_1 \big), \qquad \forall\, x_i \in \mathcal{S} .
    \label{eq:min_update}
\end{equation}
Maintaining $\mathbf{d}$ thus lets each candidate be scored by a single $O(ND)$ update, rather than the $O(NkD)$ of recomputing all nearest distances from scratch.

\paragraph{Greedy construction.}
As summarized in Algorithm~\ref{alg:coreset_selection}, we optimize
Eq.~\ref{eq:objective} greedily. Starting from $\mathcal{S}'=\varnothing$, each iteration scores every candidate $x_j$ in the remaining pool by its marginal gain,
\begin{equation}
    g_j(\mathcal{S}') \;=\; \bar{d}(\mathcal{S}') - \bar{d}\big(\mathcal{S}'\cup\{x_j\}\big)
    \;=\; \frac{1}{N}\sum_{i=1}^{N}
          \big( d_i(\mathcal{S}') - \lVert \mathbf{e}_{i}-\mathbf{e}_{j}\rVert_1 \big)_{+} ,
    \label{eq:gain}
\end{equation}
where $(\cdot)_{+}=\max\{0,\cdot\}$ and the second equality follows from Eq.~\ref{eq:min_update}, and permanently adds the maximizer. Because appending a sample can only decrease each $d_i$, every $g_j$ is non-negative (monotonicity) and non-increasing across iterations (submodularity), the diminishing-returns property proved in Lemma~\ref{lem:submod}. A gain computed at an earlier iteration is therefore always an upper bound on the current one, which the lazy greedy variant of \citep{minoux2005accelerated} exploits: the gains are kept in a max-heap $\mathcal{Q}$, and each iteration re-evaluates only the top element, re-inserting it if its exact gain falls below the next key (Algorithm~\ref{alg:coreset_selection}, lines~10--19). A recomputed top gain that still dominates the heap is provably the exact maximizer, so most candidate evaluations are skipped while the selection remains identical to naive greedy.
 
\paragraph{Global optimality guarantee.}
In Appendix~\ref{sec:guarantee}, we prove that the objective in Eq.~\ref{eq:objective} is monotone submodular, so that the greedy algorithm (Algorithm~\ref{alg:coreset_selection}) attains at least $(1-1/e)$ of the maximum coverage achievable under the same budget \citep{nemhauser1978analysis}, a dataset-independent,
\emph{a priori} lower bound that holds at every coreset size.

\section{Experiments}
\label{exp}
We assessed the performance of GCcore on 10 datasets (dataset details in Appendix~\ref{Datasets}) covering three common yet distinct computational pathology tasks: WSI classification, tile classification and tissue segmentation. Collectively, these datasets span prostate, colorectal, breast, lung and gastric tissue, as well as a pan-cancer cohort, with several datasets collected across multiple medical centers.

\subsection{Experimental setup}
\label{exp_setup}
For each dataset, we first extracted embeddings for all images using an FM. Coresets of varying sizes were then selected from the embeddings of the full training set, and used to train a model comprising a frozen FM as the encoder followed by a trainable task-specific head. Performance was evaluated on the held-out test set. Unless otherwise specified, test performance is measured using QWK for ordinal multiclass classification, balanced accuracy for binary classification, and weighted Dice for segmentation, following the community standard~\citep{ding2025multimodal}. For each dataset, we defined four coreset sizes in the low-retention regime (0.01\%-20\% depending on the dataset, as detailed in Appendix~\ref{dataset_summary}), corresponding to the selection of only a few tens to a few hundred samples for annotation, a quantity that could realistically be annotated by a pathologist in clinical practice.

\paragraph{Baselines} We compared our deterministic GCcore against 14 label-free baselines. For each baseline we sweep a single characteristic hyperparameter over five values, sampled within the range reported as robust on natural-image benchmarks or around the reported default value, holding the remainder at their default settings, and report the mean test performance across the five configurations. We sweep rather than fix a value because the defaults and robust ranges of these hyperparameters, established on natural images, generally do not transfer reliably to histopathology (see Appendix~\ref{app:baseline_hparams}), so any single value could favor or penalize a baseline. Averaging over five values instead estimates the performance a practitioner can expect on a new cohort without labels for tuning, and the standard deviation (SD) quantifies the sensitivity of the baseline to its own hyperparameter. 


(1) \textbf{Random} selection, across random seeds.
(2) \textbf{K-center}~\citep{sener2017active}, across random seeds.
(3) \textbf{K-means}~\citep{macqueen1967multivariate}, across the number of clusters $M$. The budget is allocated to clusters in proportion to their size, and within each cluster we perform radial stratified sampling at evenly spaced distances from the centroid, so that the subset spans each cluster from its core to its periphery.
(4) \textbf{Proto}~\citep{sorscher2022beyond}, across the number of clusters $M$. We define the \emph{Proto score} of a sample as the cosine distance to its assigned cluster centroid, so larger values indicate less prototypical samples. We retain either the hardest (least prototypical) samples, corresponding to the data-abundant regime, or the easiest (most prototypical) ones, corresponding to the data-scarce regime, denoted \textbf{Proto-Hard} and \textbf{Proto-Easy}.
(5) \textbf{Moderate}~\citep{xia2022moderate}, across the number of clusters $M$, retaining the samples whose Proto score lies closest to the per-cluster median.
(6) \textbf{CCS}~\citep{zheng2022coverage}, across the hard-cutoff rate $\beta$. Using the Proto score as difficulty, CCS prunes the hardest $\beta$ fraction, stratifies the remainder into equal-width difficulty bins, and allocates the budget to the least-populated bins first.
(7) \textbf{D\textsuperscript{2}}~\citep{maharana2024mathbb}, across the number of nearest neighbors $k_{\mathrm{nn}}$ in the graph, with a uniform initial difficulty score (\textbf{D\textsuperscript{2}-Uniform}), and with the Proto score (\textbf{D\textsuperscript{2}-Proto}).
(8) \textbf{ELFS}~\citep{zheng2025elfs}, performing double-end pruning with either the AUM~\citep{pleiss2020identifying} of a pseudo-label linear probe (\textbf{ELFS-Pseudo}), swept across the number of pseudo-classes $\tilde{C}$, or the Proto score (\textbf{ELFS-Proto}), swept across the number of clusters $M$.
(9) \textbf{InfoMax}~\citep{tan2025data}, across the pairwise weight $\alpha$, with the Proto score as the per-sample informativeness term: higher scores are more informative in hard mode (\textbf{InfoMax-Hard}), lower in easy mode (\textbf{InfoMax-Easy}).
(10) \textbf{ZCore}~\citep{griffin2026zero}, across random seeds.

\begin{table}[t]
\centering
\caption{%
Mean rank of the 15 coreset selection methods on each dataset; the subscript denotes the SD of that rank across five hyperparameter values. GCcore is deterministic and reports no SD. \textbf{Mean}: average across the 10 datasets.
}
\label{tab:rank_sd}
\small
\setlength{\tabcolsep}{0pt}
\begin{tabular*}{\linewidth}{@{\extracolsep{\fill}} l ccc cccc ccc c @{}}
\toprule
\multirow[c]{3}{*}{\makecell[l]{Selection\\ method}}
& \multicolumn{3}{c}{WSI classification} & \multicolumn{4}{c}{Tile classification}
& \multicolumn{3}{c}{Segmentation} & \multirow[c]{3}{*}{\textbf{Mean}} \\
\cmidrule(lr){2-4} \cmidrule(lr){5-8} \cmidrule(lr){9-11}
& PANDA & \makecell{IMP-\\CRS} & \makecell{TCGA-\\BRCA} & \makecell{WSSS4\\LUAD}
& SICAP & \makecell{CRC-\\100k} & \makecell{Deep\\HP} & \makecell{BEE\\TLE}
& panCK & \makecell{CAME-\\LYON16} & \\
\midrule
GCcore       & $3.0$ & $\underline{\mathbf{1.0}}$ & $\underline{\mathbf{2.5}}$ & $\underline{\mathbf{1.0}}$ & $2.5$ & $3.0$ & $\underline{\mathbf{1.0}}$ & $\underline{\mathbf{1.3}}$ & $5.8$ & $\underline{\mathbf{1.3}}$ & $\underline{\mathbf{2.2}}$ \\
Random       & \ms{6.5}{1.8} & \ms{4.8}{1.2} & \ms{2.8}{1.2} & \ms{5.5}{1.9} & \ms{3.3}{1.9} & \ms{5.0}{1.1} & \ms{3.5}{0.9} & \ms{2.3}{1.6} & \ms{4.0}{1.6} & \ms{5.5}{1.1} & \ms{4.3}{1.4} \\
K-means      & \msb{2.0}{0.8} & \ms{3.0}{1.4} & \ms{5.8}{2.7} & \ms{6.6}{1.6} & \ms{6.0}{3.1} & \msb{1.3}{0.9} & \ms{5.8}{0.7} & \ms{3.5}{0.8} & \ms{5.0}{2.0} & \ms{7.0}{1.8} & \ms{4.6}{1.6} \\
ELFS-Proto   & \ms{5.8}{1.3} & \ms{4.8}{1.4} & \ms{6.8}{3.1} & \ms{8.8}{1.4} & \msb{2.0}{3.0} & \ms{7.3}{1.0} & \ms{3.8}{0.5} & \ms{4.3}{1.2} & \ms{8.5}{1.7} & \ms{4.8}{0.8} & \ms{5.7}{1.5} \\
K-center     & \ms{3.0}{1.5} & \ms{6.3}{0.8} & \ms{11.3}{4.5} & \ms{6.5}{3.0} & \ms{4.8}{1.7} & \ms{4.0}{1.5} & \ms{8.3}{0.4} & \ms{6.5}{0.4} & \msb{3.3}{0.6} & \ms{2.8}{0.7} & \ms{5.7}{1.5} \\
Moderate     & \ms{7.8}{2.8} & \ms{9.0}{1.5} & \ms{4.5}{2.4} & \ms{9.1}{2.6} & \ms{5.8}{1.4} & \ms{4.3}{1.9} & \ms{4.0}{1.0} & \ms{3.8}{1.2} & \ms{8.3}{2.8} & \ms{7.3}{2.7} & \ms{6.4}{2.0} \\
ELFS-Pseudo  & \ms{4.5}{1.7} & \ms{6.0}{2.4} & \ms{10.3}{3.4} & \ms{4.0}{1.4} & \ms{8.3}{2.6} & \ms{8.0}{1.5} & \ms{7.0}{0.8} & \ms{6.8}{0.9} & \ms{9.0}{2.2} & \ms{4.0}{1.1} & \ms{6.8}{1.8} \\
CCS          & \ms{7.3}{0.9} & \ms{10.8}{2.7} & \ms{8.8}{2.6} & \ms{6.5}{1.2} & \ms{6.0}{1.7} & \ms{8.5}{2.3} & \ms{3.5}{0.2} & \ms{8.5}{3.3} & \ms{6.3}{2.1} & \ms{11.3}{3.8} & \ms{7.7}{2.1} \\
InfoMax-Easy & \ms{8.8}{3.3} & \ms{10.3}{0.3} & \ms{10.3}{1.9} & \ms{7.3}{3.8} & \ms{10.8}{0.7} & \ms{6.5}{3.3} & \ms{8.5}{1.7} & \ms{8.5}{1.5} & \ms{6.8}{3.8} & \ms{6.8}{0.3} & \ms{8.4}{2.0} \\
ZCore        & \ms{11.5}{1.3} & \ms{8.8}{0.8} & \ms{11.8}{1.0} & \ms{9.5}{2.5} & \ms{10.3}{0.9} & \ms{7.3}{2.0} & \ms{11.0}{0.9} & \ms{10.3}{0.4} & \ms{5.5}{1.9} & \ms{6.0}{2.8} & \ms{9.2}{1.4} \\
InfoMax-Hard & \ms{11.0}{1.4} & \ms{5.5}{1.5} & \ms{10.3}{1.7} & \ms{5.0}{1.2} & \ms{12.0}{1.7} & \ms{11.0}{2.0} & \ms{13.0}{0.4} & \ms{10.5}{0.5} & \ms{9.3}{1.4} & \ms{9.8}{2.3} & \ms{9.7}{1.4} \\
D\textsuperscript{2}-Proto     & \ms{9.3}{2.7} & \ms{9.0}{2.7} & \ms{6.3}{1.1} & \ms{8.8}{2.2} & \ms{12.0}{2.1} & \ms{12.8}{0.4} & \ms{13.5}{1.1} & \ms{13.5}{0.3} & \ms{6.5}{2.4} & \ms{11.8}{3.9} & \ms{10.3}{1.9} \\
D\textsuperscript{2}-Uniform   & \ms{11.0}{0.9} & \ms{14.3}{0.1} & \ms{4.3}{3.5} & \ms{12.8}{4.5} & \ms{7.5}{1.7} & \ms{14.8}{0.4} & \ms{11.8}{1.0} & \ms{14.3}{0.4} & \ms{14.0}{1.9} & \ms{14.9}{0.2} & \ms{11.9}{1.5} \\
Proto-Hard   & \ms{14.8}{1.6} & \ms{12.0}{1.4} & \ms{14.5}{2.5} & \ms{13.8}{2.0} & \ms{14.5}{3.2} & \ms{12.3}{0.4} & \ms{14.3}{0.5} & \ms{12.0}{0.5} & \ms{13.0}{1.8} & \ms{13.0}{1.5} & \ms{13.4}{1.5} \\
Proto-Easy   & \ms{14.0}{2.1} & \ms{14.8}{0.2} & \ms{10.3}{2.0} & \ms{15.0}{0.3} & \ms{14.5}{1.1} & \ms{14.3}{0.2} & \ms{11.3}{1.3} & \ms{14.3}{0.1} & \ms{15.0}{0.1} & \ms{14.1}{1.1} & \ms{13.7}{0.9} \\
\bottomrule
\end{tabular*}
\end{table}


\paragraph{WSI classification}
\label{exp:wsi-classification}
For WSI classification, we conducted experiments on PANDA~\citep{bulten2022artificial}, IMP-CRS~\citep{oliveira2021cad, neto2022imil4path, neto2024interpretable}, and TCGA-BRCA datasets. We used the PRISM slide encoder~\citep{shaikovski2024prism} to extract the slide embeddings of WSIs. The classification head is a single linear layer that maps a slide embedding $\mathbf{e} \in \mathbb{R}^{D}$ to $C$ logits, where $D$ is the native dimensionality of the slide encoder and $C$ is the number of classes defined by the dataset label. Since the classes are mutually exclusive, the logits are normalized by a softmax and the head is trained with class-weighted cross-entropy loss.

\paragraph{Tile classification}
\label{exp:tile-classification}
The tile classification task encompasses experiments on four datasets, namely WSSS4LUAD~\citep{han2022wsss4luad}, SICAP~\citep{silva2020going}, CRC-100k~\citep{kather2018100}, and DeepHP~\citep{goncalves2022deephp, goncalves2023deephp}. We used the H-optimus-1 tile encoder~\citep{scalbert2026hoptimus1} to extract the embeddings of tiles. For SICAP, CRC-100k and DeepHP, the classification head is identical to that of WSI classification, except that $D$ is the dimensionality of the tile encoder. In WSSS4LUAD the tile labels are not mutually exclusive,
since one tile may contain several tissue types; we therefore replace the softmax with $C=3$ independent sigmoids applied to the logits of the same linear layer and train with binary cross-entropy loss. Performance on WSSS4LUAD is measured using macro AUROC, averaged over the three classes. Performance on CRC-100k is measured using balanced accuracy, since its nine tissue classes are nominal and QWK presumes an ordinal label scale.

\paragraph{Segmentation}
\label{exp:tissue-segmentation}
We experimented on BEETLE~\citep{lems2025multicentric}, panCK~\citep{komura2023restaining}, and CAMELYON16~\citep{ehteshami2017diagnostic} for segmentation. For coreset selection we use the global [CLS] token of H-optimus-1~\citep{scalbert2026hoptimus1} as the representation of each tile, as in the tile-classification experiments. For training and evaluation, however, the segmentation head instead uses the patch tokens, reshaped into a two-dimensional grid that retains the spatial correspondence between each token and the region it covers within the input tile. The segmentation head comprises four hierarchical upsampling stages producing pixel-wise class predictions at the working resolution, as detailed in Appendix~\ref{seg_head}. It is trained with class-weighted cross-entropy, with unannotated pixels, the ignore class in BEETLE and background in CAMELYON16, excluded from both the loss and the evaluation metrics.

\begin{table}[t]
\centering
\caption{%
Test performance of the 15 selection methods at the lowest retention rate (0.01\%--5\%, depending on the dataset; see Appendix~\ref{dataset_summary}) resulting in only a few tens of samples selected for annotation; the subscript is the SD across five hyperparameter values. All values are multiplied by 100. GCcore is deterministic and reports no SD. D\textsuperscript{2}-Uniform on DeepHP is marked because its selected subset contains samples from only one of the two classes.}
\label{tab:acc_lowest}
\small
\setlength{\tabcolsep}{0pt}
\begin{tabular*}{\linewidth}{@{\extracolsep{\fill}} l ccc cccc ccc @{}}
\toprule
\multirow[c]{2}{*}{\makecell[l]{Selection\\ method}}
& \multicolumn{3}{c}{WSI classification} & \multicolumn{4}{c}{Tile classification}
& \multicolumn{3}{c}{Segmentation} \\
\cmidrule(lr){2-4} \cmidrule(lr){5-8} \cmidrule(lr){9-11}
& PANDA & \makecell{IMP-\\CRS} & \makecell{TCGA-\\BRCA} & \makecell{WSSS4\\LUAD}
& SICAP & \makecell{CRC-\\100k} & \makecell{Deep\\HP} & \makecell{BEE\\TLE}
& panCK & \makecell{CAME-\\LYON16} \\
\midrule
GCcore       & $73.3$ & $\underline{\mathbf{91.7}}$ & $85.9$ & $\underline{\mathbf{94.8}}$ & $66.8$ & $\underline{\mathbf{86.6}}$ & $\underline{\mathbf{80.0}}$ & $\underline{\mathbf{84.4}}$ & $\underline{\mathbf{76.2}}$ & $\underline{\mathbf{97.8}}$ \\
Random       & \ms{70.8}{5.8} & \ms{87.4}{1.9} & \msb{89.8}{3.0} & \ms{93.5}{1.2} & \ms{62.9}{8.7} & \ms{74.6}{9.0} & \ms{75.0}{3.4} & \ms{81.6}{2.7} & \ms{75.9}{0.9} & \ms{94.9}{1.6} \\
K-means      & \ms{73.7}{3.2} & \ms{86.2}{2.9} & \ms{88.0}{2.9} & \ms{92.2}{1.3} & \ms{68.6}{2.8} & \ms{85.4}{5.9} & \ms{71.7}{4.4} & \ms{84.2}{1.5} & \ms{75.8}{0.7} & \ms{91.2}{4.4} \\
ELFS-Proto   & \msb{74.2}{3.1} & \ms{85.8}{2.3} & \ms{86.4}{2.8} & \ms{91.8}{1.3} & \msb{69.7}{5.3} & \ms{74.7}{6.0} & \ms{72.7}{4.2} & \ms{80.1}{3.6} & \ms{75.3}{0.7} & \ms{95.5}{2.5} \\
K-center     & \ms{71.5}{4.3} & \ms{88.4}{1.1} & \ms{79.9}{8.5} & \ms{92.7}{2.1} & \ms{63.4}{8.5} & \ms{77.0}{9.0} & \ms{65.5}{2.6} & \ms{76.5}{2.1} & \ms{74.4}{0.8} & \ms{96.8}{0.6} \\
Moderate     & \ms{68.1}{3.3} & \ms{85.8}{4.1} & \ms{85.5}{2.8} & \ms{92.9}{2.1} & \ms{62.2}{4.1} & \ms{80.2}{7.1} & \ms{66.6}{7.3} & \ms{81.6}{2.7} & \ms{74.6}{2.5} & \ms{92.2}{3.2} \\
ELFS-Pseudo  & \ms{68.1}{9.2} & \ms{82.6}{7.7} & \ms{84.3}{4.5} & \ms{92.9}{1.6} & \ms{48.2}{10.4} & \ms{81.4}{4.9} & \ms{69.4}{4.5} & \ms{79.5}{5.6} & \ms{74.9}{0.8} & \ms{96.5}{2.4} \\
CCS          & \ms{71.7}{2.9} & \ms{85.6}{2.5} & \ms{80.9}{5.8} & \ms{93.4}{0.6} & \ms{67.2}{7.1} & \ms{73.6}{8.9} & \ms{74.4}{2.3} & \ms{76.1}{10.1} & \ms{75.9}{0.6} & \ms{90.8}{6.3} \\
InfoMax-Easy & \ms{62.4}{9.3} & \ms{76.7}{10.3} & \ms{84.3}{0.4} & \ms{94.0}{3.3} & \ms{51.3}{2.4} & \ms{75.1}{16.2} & \ms{70.8}{4.6} & \ms{79.0}{5.9} & \ms{69.8}{8.7} & \ms{96.3}{0.4} \\
ZCore        & \ms{60.1}{6.4} & \ms{58.9}{8.3} & \ms{76.4}{10.2} & \ms{88.7}{4.1} & \ms{36.6}{3.5} & \ms{64.9}{12.8} & \ms{60.3}{11.3} & \ms{67.6}{4.0} & \ms{68.7}{4.1} & \ms{93.5}{5.8} \\
InfoMax-Hard & \ms{47.3}{24.6} & \ms{83.7}{5.3} & \ms{81.1}{5.8} & \ms{93.4}{3.1} & \ms{37.2}{20.4} & \ms{44.2}{19.7} & \ms{52.6}{4.8} & \ms{71.1}{3.4} & \ms{70.9}{1.8} & \ms{91.8}{5.3} \\
D\textsuperscript{2}-Proto     & \ms{51.5}{11.9} & \ms{51.0}{31.7} & \ms{82.6}{4.5} & \ms{86.5}{5.9} & \ms{35.8}{8.2} & \ms{15.3}{3.6} & \ms{62.5}{9.5} & \ms{36.8}{15.1} & \ms{73.9}{2.8} & \ms{86.3}{4.8} \\
D\textsuperscript{2}-Uniform   & \ms{53.2}{10.8} & \ms{-8.4}{12.7} & \ms{87.2}{4.5} & \ms{89.7}{2.8} & \ms{53.0}{16.3} & \ms{8.5}{1.6} & -- & \ms{44.8}{2.5} & \ms{63.1}{5.0} & \ms{84.1}{0.0} \\
Proto-Hard   & \ms{37.7}{14.7} & \ms{52.1}{31.3} & \ms{74.5}{9.7} & \ms{75.2}{12.3} & \ms{19.0}{15.9} & \ms{17.1}{3.7} & \ms{51.3}{1.5} & \ms{48.3}{4.3} & \ms{63.7}{3.7} & \ms{84.5}{0.8} \\
Proto-Easy   & \ms{40.3}{18.4} & \ms{-12.9}{21.6} & \ms{81.2}{5.2} & \ms{57.6}{26.4} & \ms{34.8}{10.9} & \ms{11.9}{2.4} & \ms{61.9}{2.6} & \ms{46.3}{0.0} & \ms{53.4}{8.6} & \ms{84.1}{0.0} \\
\bottomrule
\end{tabular*}
\end{table}

\subsection{Main results}
\label{sec:main_results}

Tables~\ref{tab:rank_sd} and~\ref{tab:acc_lowest} show that no single baseline is consistently competitive. GCcore achieves the best mean rank over the 10 datasets, 2.2 out of 15, ahead of random selection (4.3), K-means (4.6), and ELFS-Proto and K-center (both 5.7). It ranks first on six of the 10 datasets and within the top three on nine. Every baseline that is competitive on average is weak somewhere: K-means is best on CRC-100k (1.3) but reaches only 7.0 on CAMELYON16, K-center is best on panCK (3.3) but 11.3 on TCGA-BRCA, and ELFS-Proto is best on SICAP (2.0) but 8.8 on WSSS4LUAD. Table~\ref{tab:acc_lowest} shows the same asymmetry in test performance at the lowest retention rate of each dataset. 
GCcore achieves the best test performance on seven of the 10 datasets and trails the best method by at most 3.9 points on the remaining three, whereas a competitive baseline that outperforms GCcore on one dataset can be behind GCcore by up to 11.9 points on another (Random, $+3.9$ on TCGA-BRCA against $-11.9$ on CRC-100k; ELFS-Proto, $+2.9$ on SICAP against $-11.9$ on CRC-100k; K-means, $+2.2$ on TCGA-BRCA against $-8.3$ on DeepHP). 

The subscripts in both tables quantify the sensitivity of a method to its own configuration. For each dataset the subscript is the SD, across five hyperparameter values, of the rank (Table~\ref{tab:rank_sd}) or of the test performance (Table~\ref{tab:acc_lowest}) the method obtains. Because GCcore selects its coreset as a deterministic function of the embeddings and exposes no such hyperparameter, this SD is zero on every dataset, whereas the rank SD of a baseline averages 0.9 to 2.1 positions across datasets and reaches 4.5 positions on a single dataset. This sensitivity is not predicted by how well a baseline ranks: across the 14 baselines, the Spearman correlation between mean rank and mean rank SD is weak and not significant ($\rho = -0.21$, $p = 0.48$). A baseline that ranks well on average is therefore not necessarily one whose rank can be trusted from a single configuration.
Figure~\ref{fig:combined_box} resolves the test-performance subscripts of Table~\ref{tab:acc_lowest} per retention rate, with GCcore drawn as a single level over each group and each box spanning one SD of a baseline. It further illustrates that a baseline mean above the GCcore level does not imply a reliable advantage. At 0.5\% retention on PANDA, K-means with number of clusters $M = 6, 10, 15, 20, 25$ yields 72.3, 71.7, 73.8, 79.2 and 71.3, a mean of 73.7 against 73.3 for GCcore, yet three of the five K-means configurations fall below GCcore and the mean is carried by a single setting. Choosing the value of $M$ that realizes the advantage requires held-out labels, which is what a clinical low-label setting precludes. See Appendix~\ref{GC_deter} for further details.

\paragraph{Computational efficiency}
\label{sec:computational_efficiency}

We measured the selection runtime of GCcore and 14 baselines on a single machine (AMD Ryzen Threadripper PRO 3995WX, 64 cores / 128 threads, 256~GB RAM) using CRC-100k, averaged over five runs. Under low retention rates (20--200 tiles retained from a 100{,}000-tile training set), GCcore averages 86.4~s and is faster than 12 of the 14 baselines, behind only Random (0.0002~s) and K-center (0.49~s). Selecting 100 tiles, for instance, takes GCcore 88.9~s, against 0.0002~s for random and 0.53~s for K-center. In return, GCcore attains a mean balanced accuracy of 0.905 across the four budgets, against 0.868 for random and 0.875 for K-center.


\subsection{Ablation studies}
We investigate the robustness of GCcore to the experimental design and to the data: the choice of FM for extracting the embeddings, a discrepancy between the FM used for coreset selection and the FM used for downstream model training, class imbalance in the dataset, the distance used to instantiate the objective,
and the embedding dimensionality. Figure~\ref{fig:ablation_overview} summarizes the five studies and Appendix~\ref{ablation_FMchoice}--\ref{ablation_dim} give the experimental details and the analysis for each.
(a)~\textbf{FM choice (Table~\ref{tab:rank_sd_fm}).} GCcore is the best-ranked of the 15 selection methods under five of the six feature extractors and is never worse than rank 3.0, giving a mean rank of 1.5 against 3.9 for the strongest baseline; the advantage is therefore a property of the selection method rather than of any particular encoder.
(b)~\textbf{Selection vs.\ training FM (Table~\ref{tab:rank_select_encoder}).}
GCcore attains the best mean rank of all methods (2.7); whenever a stronger FM (PRISM or TITAN) provides the selection embedding, it is among the two best selection methods irrespective of which FM is used for training, so a coreset need not be recomputed for a different downstream encoder.
(c)~\textbf{Class imbalance (Table~\ref{tab:imb_rank_sd}).} Across three simulated imbalanced training sets (imbalance ratio 55.7--231.3) GCcore has the best mean rank (3.3) and is the only strategy placed among the two best on all three; on IMB2 it recovers 90.7\% of the full-data test performance from 19 WSIs, which
Random does not match until 751 WSIs.
(d)~\textbf{Distance metric (Table~\ref{tab:rank_sd_l1l2}).} GCcore with the $\ell_1$ and with the $\ell_2$ distance both rank ahead of every baseline on average (2.7 and 3.4 among 16, against 5.0 for the strongest baseline), so the advantage follows from the global coverage maximization mechanism rather than from the distance used to instantiate it. Note that the guarantee of Eq.~\ref{eq:guarantee} holds for any distance metric.
(e)~\textbf{Embedding dimensionality (Table~\ref{tab:rank_dim}).} After principal component analysis (PCA) retaining 99\% of the explained variance, which removes 73--85\% of the dimensions, GCcore remains the best-ranked method (2.5), suggesting that GCcore can be applied directly to compressed slide representations without re-tuning.

\begin{figure}[t]
  \centering
  \includegraphics[width=\linewidth]{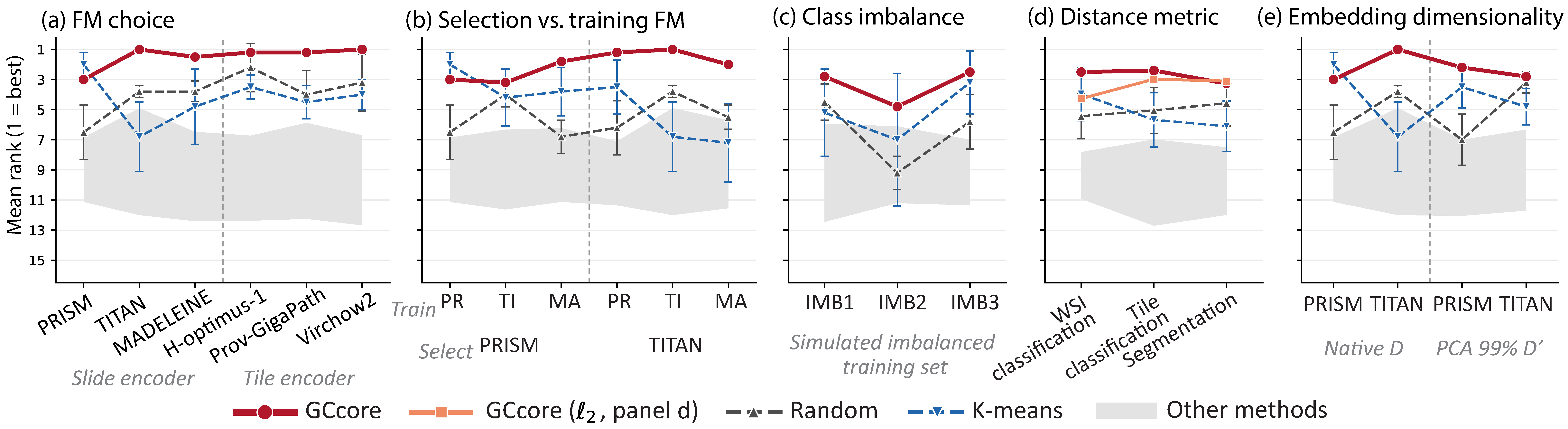}
  \caption{Ablation studies over the experimental design and the data: FM choice, FM choice for selection and for model training, class imbalance in the data, distance metric used in GCcore, and dimensionality of the FM embeddings.
  Experiments in (a) are conducted on PANDA and BEETLE, those in (b), (c) and (e) on PANDA, and those in (d) on all 10 datasets.}
  \label{fig:ablation_overview}
\end{figure}

\section{Conclusion}
To optimize the annotation and training process of DL models in computational pathology, we proposed a novel hyperparameter-free coreset selection algorithm, GCcore, that exploits embeddings from any histopathology FM to identify representative samples within heterogeneous and unlabeled histological datasets. Overall, at clinically realistic retention rates, our approach outperformed existing methods and showed a more consistent performance across different cancer types and tasks. GCcore therefore provides a reliable and reproducible strategy for optimal data selection, identifying informative subsets that support downstream training while substantially reducing the required labeling effort.

\subsection*{Ethics statement}
All datasets used in this work are publicly available datasets obtained from established public sources. We do not collect any new data involving human subjects.

\bibliography{iclr2027_conference}
\bibliographystyle{iclr2027_conference}

\newpage
\appendix

\section{Pseudo-code}
\vspace*{0.3cm}

\begin{algorithm*}[h]
\caption{Coreset Selection by $\mathrm{GC}$ Maximization via Mean Nearest-Neighbor Distance}
\label{alg:coreset_selection}
\begin{algorithmic}[1]
\algnewcommand{\Inputs}[1]{\State \textbf{Inputs:} #1}
\algnewcommand{\Outputs}[1]{\State \textbf{Outputs:} #1}
\vspace{0.2em}
\Inputs
$\mathcal{S} = \{ x_i \}_{i=1}^{N}$: unlabeled dataset; \; $f(\cdot)$: foundation model; \; $k$: target coreset size.
\Outputs
Coreset $\mathcal{S}' \subset \mathcal{S}$ with $|\mathcal{S}'| = k$, and its coverage $\mathrm{GC}(\mathcal{S}')$.
\vspace{0.2em}
\Statex \vspace{0.3em} \hrule \vspace{0.1em}

\State For each $x_i \in \mathcal{S}$, $\mathbf{e}_{i} \gets f(x_i)$; \quad $\mathbf{E} \gets \{ \mathbf{e}_{i} \}_{i=1}^{N} \subset \mathbb{R}^{D}$
    \Comment{Extract $D$-dimensional embeddings}
\State $r_{\max} \gets \max_{i<j} \| \mathbf{e}_{i} - \mathbf{e}_{j} \|_1$
    \Comment{Max pairwise distance; ensures $\mathrm{sim}\!\ge\!0$ (Appendix~\ref{sec:guarantee})}
\State $\mathbf{d} \gets (r_{\max}, \dots, r_{\max}) \in \mathbb{R}^{N}$; \quad $\bar{d} \gets r_{\max}$
    \Comment{Invariant: $\mathbf{d} = \mathbf{d}(\mathcal{S}')$, $\bar{d} = \bar{d}(\mathcal{S}')$; $d_i(\varnothing) \!=\! r_{\max}$}
\State $(\mathcal{S}', \mathcal{R}) \gets (\varnothing, \mathcal{S})$
\For{each $x_j \in \mathcal{R}$}
    \State $g_j \gets \frac{1}{N} \sum_{i=1}^{N} \big( d_i - \| \mathbf{e}_{i} - \mathbf{e}_{j} \|_1 \big)_{+}$
        \Comment{Initial reduction in $\bar{d}$ from adding $x_j$}
\EndFor
\State $\mathcal{Q} \gets$ max-heap over $\{ (g_j, x_j) \}_{x_j \in \mathcal{R}}$
    \Comment{Lazy-greedy priority queue}
\While{$|\mathcal{S}'| < k$}
    \Repeat
        \State $(g_j, x_j) \gets \mathcal{Q}.\textsc{PopMax}()$
        \State $g_j \gets \frac{1}{N} \sum_{i=1}^{N} \big( d_i - \| \mathbf{e}_{i} - \mathbf{e}_{j} \|_1 \big)_{+}$
            \Comment{Recompute exact gain at current $\mathbf{d}$}
        \State $\mathrm{valid} \gets \big( \mathcal{Q} = \varnothing \big) \,\vee\, \big( g_j \ge \mathcal{Q}.\textsc{PeekKey}() \big)$
        \If{$\neg\,\mathrm{valid}$} \State $\mathcal{Q}.\textsc{Push}(g_j, x_j)$ \EndIf
    \Until{$\mathrm{valid}$}
    \State $d_i \gets \min\!\big( d_i,\, \| \mathbf{e}_{i} - \mathbf{e}_{j} \|_1 \big), \;\; \forall\, x_i \in \mathcal{S}$
        \Comment{Eq.~\ref{eq:min_update}; incremental coverage update}
    \State $\bar{d} \gets \bar{d} - g_j$; \quad $(\mathcal{S}', \mathcal{R}) \gets (\mathcal{S}' \cup \{ x_{j} \},\, \mathcal{R} \setminus \{ x_{j} \})$
\EndWhile

\State \Return $\mathcal{S}', \;\; \mathrm{GC}(\mathcal{S}') = 1 - \bar{d} / r_{\max}$
    \Comment{Eq.~\ref{eq:equiv}}
\end{algorithmic}
\end{algorithm*}

\newpage
\section{Notational convention}
\label{notation}
\vspace*{0.3cm}
\begin{table}[h]
\centering
\caption{Comprehensive overview of the notational convention.}
\label{tab:notation}
\begin{tabularx}{\columnwidth}{@{}lX@{}}
\toprule
\textbf{Notation} & \textbf{Description} \\
\midrule
$\mathcal{S}$ & The unlabeled training set, with $\lvert \mathcal{S}\rvert = N$. \\
$x_i \in \mathcal{S}$ & A training sample. \\
$N$ & The number of samples in $\mathcal{S}$. \\
$f(\cdot)$ & The pre-trained FM used as a frozen feature extractor. \\
$\mathbf{E} = \{\mathbf{e}_i\}_{i=1}^{N}$ & The set of embeddings of the full training set. \\
$\mathbf{e}_i = f(x_i) \in \mathbb{R}^{D}$ & The embedding of sample $x_i$. \\
$D$ & The embedding dimension. \\
$\mathcal{S}' \subset \mathcal{S}$ & The coreset selected by GCcore, with $\lvert \mathcal{S}'\rvert = k$. \\
$k$ & The annotation budget and the target coreset size. \\
$\mathcal{T} \subseteq \mathcal{S}$ & An arbitrary subset of $\mathcal{S}$; the argument of the coverage objective. \\
$r$ & A coverage radius. \\
$d_i(\mathcal{T})$ & The distance from sample $x_i$ to its nearest representative in $\mathcal{T}$, see Eq.~(\ref{eq:di}); \\
 & by convention $d_i(\varnothing) = r_{\max}$. \\
$\mathbf{d}(\mathcal{T}) \in \mathbb{R}^{N}$ & The vector of nearest-coreset distances, $\big(d_1(\mathcal{T}),\dots,d_N(\mathcal{T})\big)^{\!\top}$. \\
$e_{im}$ & The $m$-th component of $\mathbf{e}_i$. \\
$p(r; \mathcal{T})$ & The coverage at radius $r$: the fraction of samples within $r$ of $\mathcal{T}$, see Eq.~(\ref{eq:pr}). \\
$r_{\max}$ & The maximum pairwise distance over $\mathcal{S}$, $r_{\max} = \max_{i<j}\lVert \mathbf{e}_i - \mathbf{e}_j\rVert_1$. \\
$\mathrm{GC}(\mathcal{T})$ & The global coverage of $\mathcal{T}$: the area under the coverage--radius curve, \\
 & see Eq.~(\ref{eq:gc}). $\mathrm{GC} \in [0,1]$, with $\mathrm{GC}(\varnothing) = 0$ and $\mathrm{GC}(\mathcal{S}) = 1$. \\
$\mathcal{S}^{\star} \subset \mathcal{S}$ & The optimal coreset, $\mathcal{S}^{\star} = \argmax_{\lvert \mathcal{T}\rvert = k} \mathrm{GC}(\mathcal{T})$, see Eq.~(\ref{eq:objective}). \\
$r_{\text{steps}}$ & The number of sampled radii required by a direct numerical integration of $p(r;\mathcal{T})$. \\
$\bar{d}(\mathcal{T})$ & The mean nearest-coreset distance, $\bar{d}(\mathcal{T}) = \tfrac{1}{N}\sum_{i=1}^{N} d_i(\mathcal{T})$. \\
$g_j(\mathcal{S}')$ & The marginal gain of candidate $x_j$: the reduction in $\bar{d}$ from adding $x_j$ to $\mathcal{S}'$, \\
 & see Eq.~(\ref{eq:gain}). \\
$\mathcal{R}$ & The pool of candidates not yet selected at the current iteration. \\
$\mathcal{Q}$ & The max-heap of candidate gains used by the lazy greedy solver. \\
$\mathrm{sim}(x_i,x_j)$ & The nonnegative similarity $r_{\max} - \lVert \mathbf{e}_i - \mathbf{e}_j\rVert_1 \ge 0$. \\
$c_i(\mathcal{T})$ & The per-sample coverage $\max_{x_j \in \mathcal{T}}\mathrm{sim}(x_i,x_j)$, with $c_i(\varnothing) = 0$. \\
$F(\mathcal{T})$ & The total coverage $\sum_{i=1}^{N} c_i(\mathcal{T}) = N r_{\max}\,\mathrm{GC}(\mathcal{T})$. \\
$t$ & An intermediate cardinality, $t \le k$; the first $t$ samples selected. \\
$C$ & The number of classes defined by the dataset label.\\
$D'$ & The reduced embedding dimension after principal component analysis retaining 99\% of the variance. \\
$g \times g$  & The square patch-token grid produced by the final transformer block. \\
$R$  & The working resolution used for the segmentation task, i.e. the side length of the square grid at which predicted logits and ground-truth masks are aligned. \\
$c$ & A class index, $c \in \{1,\dots,C\}$. \\
$n^{+}_{c}$ & The number of positive samples of class $c$ within the selected coreset $\mathcal{S}'$. \\
$n^{-}_{c}$ & The number of negative samples of class $c$ within the selected coreset $\mathcal{S}'$. \\
$\lambda_{c}$ & The weight applied to the positive term of class $c$ in the multi-label binary cross-entropy loss, $\lambda_{c} = \min(n^{-}_{c}/n^{+}_{c},\, \lambda_{\max})$. \\
$\lambda_{\max}$ & The cap on the positive-class weight, $\lambda_{\max} = 10$. \\
\midrule
$\lVert \cdot \rVert_1$ & The $\ell_1$ (Manhattan) norm, used throughout unless stated otherwise. \\
$\lvert \cdot \rvert$ & The cardinality of a set. \\
$\mathbf{1}\{\cdot\}$ & The indicator function. \\
$(\cdot)_{+}$ & The positive part, $(a)_{+} = \max\{0, a\}$. \\
\bottomrule
\end{tabularx}
\end{table}

\newpage

\section{Hyperparameters in existing label-free selection methods}
\label{Existing_methods}


\newcolumntype{L}[1]{>{\raggedright\arraybackslash}p{#1}}
\newcolumntype{C}[1]{>{\centering\arraybackslash}p{#1}}

\newlength{\lfA}\newlength{\lfB}\newlength{\lfC}\newlength{\lfD}\newlength{\lfE}

\newcommand{\vspan}[2]{\raisebox{0pt}[0pt][0pt]{\parbox[t]{#1}{#2}}}

\begingroup
\setlength{\lfA}{0.085\linewidth}
\setlength{\lfB}{0.105\linewidth}
\setlength{\lfC}{0.140\linewidth}
\setlength{\lfD}{0.310\linewidth}
\setlength{\lfE}{0.275\linewidth}
\footnotesize
\setlength{\tabcolsep}{3pt}
\renewcommand{\arraystretch}{1.1}
\setlength{\LTcapwidth}{\linewidth}

\begin{longtable}{L{\lfA} C{\lfB} L{\lfC} L{\lfD} L{\lfE}}

\caption{Hyperparameters of coreset selection methods that are label-free by design or that can be adapted into a label-free variant.
For every method that operates on embeddings, we regard the feature-extraction step as label-free, since applying a pathology FM to obtain embeddings requires no annotations.
Hyperparameter symbols are local to each method and follow the original publications, except where a symbol would collide with the notation of this paper (Table~\ref{tab:notation}).}
\label{tab:label_free_hyperparams} \\
\toprule
\textbf{Method} &
\textbf{Label-free}\newline\textbf{by design} &
\textbf{Label-free}\newline\textbf{adaptation} &
\textbf{Hyperparameters of the}\newline\textbf{label-free method} &
\textbf{Reported values and}\newline\textbf{how they were obtained} \\
\midrule
\endfirsthead

\multicolumn{5}{l}{\itshape Table \thetable\ (continued)} \\
\toprule
\textbf{Method} &
\textbf{Label-free}\newline\textbf{by design} &
\textbf{Label-free}\newline\textbf{adaptation} &
\textbf{Hyperparameters of the}\newline\textbf{label-free method} &
\textbf{Reported values and}\newline\textbf{how they were obtained} \\
\midrule
\endhead

\midrule
\multicolumn{5}{r}{\itshape Continued on the next page} \\
\endfoot

\bottomrule
\endlastfoot

Random & Yes & --- &
$s$: random seed. & --- \\
\midrule

K-center\newline~\citep{sener2017active} & Yes & --- &
$s$: random seed. & --- \\
\midrule

K-means\newline~\citep{macqueen1967multivariate} & Yes & --- &
$s$: random seed;\newline
$M$: number of clusters. & --- \\
\midrule

Prototype\newline~\citep{sorscher2022beyond} & Yes & --- &
$M$: number of clusters;\newline
$\mathrm{mode} \in \{\textit{abundant}, \textit{scarce}\}$: pruning direction. Hard (least prototypical) samples are preferred when the retained set is large (\textit{abundant}), and easy (prototypical) samples when it is small (\textit{scarce}). &
$M = 1{,}000$ (ImageNet). The crossover point between the two modes is derived by solving saddle-point equations that depend on the student--teacher overlap and on the margin distribution along the teacher, both defined with respect to a labeled decision boundary. It is therefore not identifiable without labels, so we treat the pruning direction as a hyperparameter. \\
\midrule

Moderate\newline~\citep{xia2022moderate} & No & Class centers are replaced by K-means centroids. &
$M$: number of clusters. & --- \\
\midrule

CCS\newline~\citep{zheng2022coverage} & No & The label-dependent difficulty score is replaced by the prototype score. &
$M$: number of clusters used to compute the prototype score;\newline
$\beta$: hard-cutoff rate, i.e.\ the fraction of the highest-difficulty samples discarded before stratification;\newline
$B$: number of strata;\newline
$s$: random seed for the uniform draws within each stratum. &
$B = 50$ for all datasets and pruning rates (CIFAR-10); $\beta$ obtained by grid search with a step size of 0.1, separately for every dataset and every pruning rate. \\
\midrule

\multirow{2}{\lfA}{D\textsuperscript{2}-Pruning\newline~\citep{maharana2024mathbb}} &
\multirow{2}{\lfB}{\centering No} &
\textit{Variant 1}: all nodes are initialized with a uniform difficulty score. &
$k_{\mathrm{nn}}$: number of nearest neighbors in the graph;\newline
$\gamma_f$: forward message weight;\newline
$\gamma_r$: reverse message weight. &
$\gamma_f$ fixed to 1.0; $k_{\mathrm{nn}} \in \{1, 5, 10, 15\}$ and $\gamma_r \in \{0, 0.1, \dots, 1.0\}$ were swept on CIFAR-10 and CIFAR-100, and the resulting trends used to select three configurations for ImageNet-1K. \\
\cmidrule(lr){3-4}
& &
\textit{Variant 2}: nodes are initialized with the prototype score. &
$M$: number of clusters used to compute the prototype score;\newline
$k_{\mathrm{nn}}$: number of nearest neighbors in the graph;\newline
$\gamma_f$: forward message weight;\newline
$\gamma_r$: reverse message weight. &
For CIFAR-10 at a coreset ratio of 0.3: $k_{\mathrm{nn}} = 10$, $\gamma_r = 0.9$. \\
\midrule

\multirow{2}{\lfA}{ELFS\newline~\citep{zheng2025elfs}} &
\multirow{2}{\lfB}{\centering Yes} &
\textit{Variant 1}: identical to the original method; the difficulty score is the AUM score obtained from pseudo-labels. &
$\tilde{C}$: number of pseudo-classes used in deep clustering;\newline
training configuration of the model trained on pseudo-labels to compute AUM scores (epochs, batch size, learning rate, optimizer, weight decay);\newline
$\beta$: hard-cutoff rate, i.e.\ the fraction of the highest-difficulty samples discarded before stratification. &
$\tilde{C}$ is set to the number of classes $C$ in the dataset; 200 epochs, batch size 512, AdamW, learning rate $10^{-4}$, weight decay $10^{-4}$ (CIFAR-10, CIFAR-100, STL-10, ImageNet); $\beta$ obtained by grid search with a step size of 0.1 on a pseudo-label validation split, separately for every dataset and every pruning rate. \\
\cmidrule(lr){3-5}
& &
\textit{Variant 2}: the AUM score is replaced by the prototype score. &
$M$: number of clusters used to compute the prototype score;\newline
$\beta$: hard-cutoff rate. & ---
 \\
\midrule

\multirow{2}{\lfA}{InfoMax\newline~\citep{tan2025data}} &
\multirow{2}{\lfB}{\centering Yes} &
\textit{Variant 1}: the prototype score in \textit{abundant} mode (hard samples carry high intra-sample information). &
\vspan{\lfD}{%
$\alpha$: pairwise weight balancing the first- and second-order terms in Eq.~(2) of the InfoMax paper;\newline
$k_{\mathrm{nn}}$: sparsity rate, i.e.\ the neighborhood size of each sample in the similarity matrix;\newline
$N_{\mathrm{part}}$: size of each partition subset;\newline
$n_{\mathrm{iter}}$: number of iterations.} &
\vspan{\lfE}{%
$N_{\mathrm{part}} \geq 10^{6}$ samples (CC3M and CC12M); $k_{\mathrm{nn}} = 5$, $\alpha = 0.3$, $n_{\mathrm{iter}} = 20$ (CC12M).} \\
\cmidrule(lr){3-3}
& &
\textit{Variant 2}: the prototype score in \textit{scarce} mode (easy samples carry high intra-sample information). & & \\
\midrule

ZCore\newline~\citep{griffin2026zero} & Yes & --- &
$s$: random seed (selection of the sampled subspaces and initialization of the per-sample importance score);\newline
$D_{\mathrm{sub}}$: dimensionality of the sampled subspaces;\newline
$\alpha$: number of nearest neighbors penalized as redundant;\newline
$\beta$: exponent of the distance penalty;\newline
$n_{\mathrm{iter}}$: number of sampling iterations. &
$D_{\mathrm{sub}} = 2$, $\alpha = 1{,}000$, $\beta = 4$, $n_{\mathrm{iter}} \approx 10^{6}$ (CIFAR-100). \\

\end{longtable}
\endgroup

\newpage

\section{Proof}
\label{sec:guarantee}
In this section, we show that the objective in Eq.~\ref{eq:objective} is monotone submodular, so that the greedy algorithm attains a $(1-1/e)$ approximation of the optimal coreset \citep{nemhauser1978analysis}. This provides a dataset-independent, a priori lower bound on the coverage of the selected coreset. To the best of our knowledge, the baselines generally do not provide a proof of such a lower-bound guarantee.

\subsection{$\mathrm{GC}$ is monotone and submodular}
For each ordered pair $(x_i,x_j)$, define the nonnegative similarity $\mathrm{sim}(x_i,x_j)=r_{\max}-\lVert \mathbf{e}_i-\mathbf{e}_j\rVert_1\geq0$. For $\mathcal{T}\subseteq \mathcal{S}$, let the per-sample coverage be $c_i(\mathcal{T})=\max_{x_j\in \mathcal{T}}\mathrm{sim}(x_i,x_j)$ with $c_i(\varnothing)=0$, consistent with $d_i(\varnothing)=r_{\max}$, and the total coverage
$F(\mathcal{T})=\sum_{i=1}^{N}c_i(\mathcal{T})$. Since $c_i(\mathcal{T})=r_{\max}-d_i(\mathcal{T})$, we have $F(\mathcal{T})=Nr_{\max}\bigl(1-\bar{d}(\mathcal{T})/r_{\max}\bigr)$, and by Eq.~\ref{eq:equiv},
\begin{equation}
    \mathrm{GC}(\mathcal{T}) \;=\; \frac{F(\mathcal{T})}{N\,r_{\max}} .
    \label{eq:aucF}
\end{equation}
As $Nr_{\max}$ is a positive constant independent of $\mathcal{T}$, $\mathrm{GC}$ and $F$ share the same monotonicity, submodularity, and maximizers.

\begin{lemma}\label{lem:submod}
$F$, and hence $\mathrm{GC}$, is normalized, monotone, and submodular.
\end{lemma}

\begin{proof}
$F(\varnothing)=0$ gives normalization. For $\mathcal{A}\subseteq \mathcal{B}\subseteq \mathcal{S}$, each $c_i$ is a maximum over a larger set, so $c_i(\mathcal{A})\leq c_i(\mathcal{B})$ and $F(\mathcal{A})\leq F(\mathcal{B})$, giving monotonicity. For
submodularity, the per-sample marginal gain of any $x\in \mathcal{S}\setminus \mathcal{B}$ is
$c_i(\mathcal{T}\cup\{x\})-c_i(\mathcal{T})=\bigl(\mathrm{sim}(x_i,x)-c_i(\mathcal{T})\bigr)_{+}$. Since $c_i(\mathcal{A})\leq c_i(\mathcal{B})$ and $a\mapsto(a)_{+}$ is nondecreasing,
$\bigl(\mathrm{sim}(x_i,x)-c_i(\mathcal{A})\bigr)_{+}\geq\bigl(\mathrm{sim}(x_i,x)-c_i(\mathcal{B})\bigr)_{+}$; summing over $i$ yields $F(\mathcal{A}\cup\{x\})-F(\mathcal{A})\geq F(\mathcal{B}\cup\{x\})-F(\mathcal{B})$. All three properties transfer to $\mathrm{GC}$ under scaling by the positive constant $1/(Nr_{\max})$.
\end{proof}

\subsection{Approximation guarantee}
By \citep{nemhauser1978analysis}, for a normalized monotone submodular set function $h$, the greedy selection $\mathcal{T}_{h}$ of size $k$ satisfies $h(\mathcal{T}_{h})\geq(1-1/e)\max_{|\mathcal{T}|\leq k}h(\mathcal{T})$. Algorithm~\ref{alg:coreset_selection} is exactly this procedure applied to $\mathrm{GC}$, so Lemma~\ref{lem:submod} gives
\begin{equation}
    \mathrm{GC}(\mathcal{S}') \;\geq\; \bigl(1-\tfrac{1}{e}\bigr)\,
    \mathrm{GC}(\mathcal{S}^{\star}) \;\approx\; 0.632\,\mathrm{GC}(\mathcal{S}^{\star}),
    \label{eq:guarantee}
\end{equation}
where $\mathcal{S}'$ is the coreset returned by Algorithm~\ref{alg:coreset_selection} and $\mathcal{S}^{\star}$ is the optimal coreset of Eq.~\ref{eq:objective}. The selected coreset thus attains at least ${\approx}63.2\%$ of the optimal coverage, independently of the dataset or embedding geometry. As the guarantee holds at every cardinality, each intermediate coreset ($t\leq k$) is likewise $(1-1/e)$-optimal, so a single run certifies the selection at every budget. Pathologists can therefore annotate the selected samples in the order they were chosen and stop whenever the budget is exhausted, with the guarantee holding at every stopping point.

\newpage

\section{Datasets}
\label{Datasets}

\subsection{Summary}
\label{dataset_summary}
\vspace*{0.5cm}
\begin{table}[htbp]
\centering
\caption{Overview of the datasets used for evaluation. Sample counts are given as W (whole-slide images), T (tiles) or P (patients).}
\label{tab:dataset_overview}
\footnotesize
\setlength{\tabcolsep}{4pt}
\renewcommand{\arraystretch}{1.1}
\begin{tabular}{@{}lllllll@{}}
\toprule
\textbf{Task} & \textbf{Dataset} & \textbf{Tissue} & \textbf{Label} &
\makecell[l]{\textbf{Dataset}\\\textbf{size}} &
\makecell[l]{\textbf{Train}\\\textbf{set}} &
\makecell[l]{\textbf{Test}\\\textbf{set}} \\
\midrule
\multirow{4}{2cm}{WSI classification}
 & PANDA     & Prostate & ISUP score & 10,616 W & 8,492 W & 2,124 W \\
 & IMP-CRS   & Colon    & \makecell[l]{non-neoplastic,\\ low-grade, high-grade}
             & 5,333 W & 4,267 W & 1,066 W \\
 & TCGA-BRCA & Breast   & IDC, ILC & 949 P & 759 P & 190 P \\
\midrule
\multirow{7}{2cm}{Tile classification}
 & WSSS4LUAD & Lung     & \makecell[l]{tumor epithelium,\\ tumor-associated stroma,\\ normal tissue}
             & 10,091 T & 8,054 T & 2,037 T \\
 & SICAP     & Prostate & NC, GG3, GG4, GG5 & 12,081 T & 9,959 T & 2,122 T \\
 & CRC-100k  & Colon    & \makecell[l]{ADI, BACK, DEB, LYM,\\ MUC, MUS, NORM,\\ STR, TUM}
             & 107,180 T & 100,000 T & 7,180 T \\
 & DeepHP    & Stomach  & HPP, HPN & 394,926 T & 315,941 T & 78,985 T \\
\midrule
\multirow{6}{*}{Segmentation}
 & BEETLE     & Breast & \makecell[l]{invasive epithelium,\\ non-invasive epithelium,\\ necrosis, other}
             & 400,941 T & 340,200 T & 60,741 T \\
 & panCK      & \makecell[l]{10 tissue\\ types} & epithelium, other
             & 26,509 T & 21,912 T & 4,597 T \\
 & CAMELYON16 & Breast & \makecell[l]{metastatic tissue,\\ normal tissue, background}
             & 479,772 T & 365,112 T & 114,660 T \\
\bottomrule
\end{tabular}
\end{table}
\vspace*{0.5cm}

Table~\ref{tab:dataset_overview} summarizes the 10 datasets used in our experiments. Coresets are selected from the full training set of each dataset, which ranges in size from 759 patients to 365{,}112 tiles. Because a fixed retention rate translates into very different annotation workloads across this range, we define a separate set of four low retention rates per dataset, chosen so that every coreset comprises tens to several hundred samples, an amount that a pathologist could realistically annotate in clinical
practice. 

For WSI classification, we retain $\{0.5, 1, 2, 5\}\%$ of PANDA, $\{1, 2, 5, 10\}\%$ of IMP-CRS and $\{5, 10, 15, 20\}\%$ of TCGA-BRCA. For tile classification, we retain $\{0.5, 1, 2, 5\}\%$ of WSSS4LUAD, $\{0.5, 1, 2, 5\}\%$ of SICAP, $\{0.02, 0.05, 0.1, 0.2\}\%$ of CRC-100k and $\{0.01, 0.05, 0.1, 0.2\}\%$ of DeepHP. For segmentation, we retain $\{0.02, 0.05, 0.1, 0.2\}\%$ of BEETLE, $\{0.2, 0.5, 1, 2\}\%$ of panCK and $\{0.02, 0.05, 0.1, 0.2\}\%$ of CAMELYON16.

Training and test sets are split at the patient level when patient identifiers are available and at the slide level otherwise. When an official split is provided, we adopt the released split.

\subsection{Descriptions}

\paragraph{PANDA}~\citep{bulten2022artificial} consists of 10,616 H\&E formalin-fixed paraffin-embedded (FFPE) diagnostic histopathology WSIs of core-needle biopsies of prostate cancer from the Radboud University Medical Center and the Karolinska Institute, digitized at 20$\times$ magnification. Each slide is assigned an International Society of Urological Pathology (ISUP) score defining the prostate cancer grade (six-class grading task). We partitioned PANDA into label-stratified training and test sets at a ratio of 80:20, yielding 8,492 WSIs (grade 0, 2,313; grade 1, 2,133; grade 2, 1,074; grade 3, 994; grade 4, 999; grade 5, 979) in the full training set and 2,124 WSIs (grade 0, 579; grade 1, 533; grade 2, 269; grade 3, 248; grade 4, 250; grade 5, 245) in the test set.

\paragraph{IMP-CRS}~\citep{oliveira2021cad, neto2022imil4path, neto2024interpretable} contains 5,333 colorectal biopsy and polypectomy slides, retrieved from the data archive of IMP Diagnostics laboratory, Portugal, digitized at 40$\times$ magnification. Each case is assigned to one of three classes: non-neoplastic; low-grade lesions, comprising conventional adenomas with low-grade dysplasia; and high-grade lesions, comprising conventional adenomas with high-grade dysplasia, intramucosal carcinomas and invasive adenocarcinomas. We performed a label-stratified 80:20 split of the dataset into training and test subsets, yielding 4,267 WSIs (678:2,278:1,311) in the full training set and 1,066 WSIs (169:569:328) in the test set, following the non-neoplastic:low-grade:high-grade class order.

\paragraph{TCGA-BRCA} comprises H\&E FFPE diagnostic histopathology WSIs of breast invasive carcinoma from The Cancer Genome Atlas (TCGA), primarily scanned at 20$\times$ magnification. Slides are assigned to one of two classes: invasive ductal carcinoma (IDC) and invasive lobular carcinoma (ILC). After excluding patients with missing slides or other cancer types, we retained WSIs from 949 patients (IDC, 758; ILC, 191), of whom 892 were represented by a single slide and 57 by multiple. We partitioned the cohort at the patient level into label-stratified training and test subsets at a ratio of 80:20, yielding 759 patients (IDC, 606; ILC, 153) in the full training set and 190 patients (IDC, 152; ILC, 38) in the test set. For each patient, we used the mean of all available slide embeddings as the patient-level embedding.

\paragraph{WSSS4LUAD}~\citep{han2022wsss4luad} consists of lung adenocarcinoma (LUAD) image tiles from Guangdong Provincial People's Hospital (GDPH) and TCGA, all scanned at 40$\times$ magnification. Each tile carries a three-element binary label vector indicating the presence of tumor epithelium, tumor-associated stroma and normal tissue. As tile-level labels are not released for the official validation and test sets, we constructed our dataset from the original training set alone, comprising 10,091 tiles from 63 WSIs, with tile dimensions ranging from $150 \times 150$ to $300 \times 300$ pixels. We partitioned the data at the WSI level and stratified by center of origin; the split was chosen to yield an approximate 80:20 ratio of tiles, giving 8,054 tiles from 53 WSIs (GDPH, 41; TCGA, 12) in the full training set and 2,037 tiles from 10 WSIs (GDPH, 8; TCGA, 2) in the test set.

\paragraph{SICAP}~\citep{silva2020going} comprises
$512 \times 512$ pixel tiles extracted from 155 WSIs of prostate core-needle biopsies, digitized at $10\times$ magnification. Each tile is annotated with one of four Gleason grades, non-cancerous (NC), Gleason grade 3 (GG3), Gleason grade 4 (GG4), and Gleason grade 5 (GG5), yielding a four-class grading task. We adopt the official split, which partitions the dataset into 9,959 tiles from 124 WSIs for training and 2,122 tiles from 31 WSIs for testing.

\paragraph{CRC-100k}~\citep{kather2018100} consists of $224 \times 224$ pixel tiles extracted from H\&E-stained histological images of human colorectal cancer (CRC) and normal tissue, digitized at $20\times$ magnification. Each tile is assigned to one of nine classes: adipose tissue (ADI), background (BACK), debris (DEB), lymphocytes (LYM), mucus (MUC), smooth muscle (MUS), normal colon mucosa (NORM), cancer-associated stroma (STR) and colorectal adenocarcinoma epithelium (TUM). We adopted the official splits, comprising 100,000 tiles for training and 7,180 tiles for testing.

\paragraph{DeepHP}~\citep{goncalves2022deephp, goncalves2023deephp}  comprises 394,926 tiles of $224 \times 224$ pixels, extracted from gastric biopsies digitized at $20\times$ magnification, of which 111,005 are labeled Helicobacter pylori-positive (HPP) and 283,921 Helicobacter pylori-negative (HPN). We partitioned the tiles into label-stratified training and test subsets at a ratio of 80:20, yielding 315,941 tiles (HPP, 88,804; HPN, 227,137) in the full training set and 78,985 tiles (HPP, 22,201; HPN, 56,784) in the test set.

\paragraph{BEETLE}~\citep{lems2025multicentric} is a multicenter, multiscanner dataset of breast cancer biopsies and surgical resections, spanning both major histological subtypes of invasive breast cancer, all molecular subtypes and all histological grades, provided at $20\times$ magnification. From the official development set, we selected 396 WSIs from patients with pixel-level annotations of four classes: invasive epithelium, non-invasive epithelium, necrosis and other. We partitioned these WSIs at the patient level and stratified by center of origin. From each WSI we then extracted non-overlapping $224 \times 224$ pixel tiles, retaining only those in which more than 50\% of pixels were annotated. This yielded 340,200 tiles from 317 WSIs (Netherlands Cancer Institute (NKI), 90; Radboud University Medical Center (RUMC), 169; Santa Chiara Hospital (SCH), 44; TCGA, 14) in the full training set and 60,741 tiles from 79 WSIs (NKI, 23; RUMC, 42; SCH, 11; TCGA, 3) in the test set.

\paragraph{panCK}~\citep{komura2023restaining} consists of 26,509 tiles of $984 \times 984$ pixels at $40\times$ magnification, collected from FFPE tissue of 407 patients and covering 10 tissue types and 16 tumor types. Each tile is provided with pixel-level annotations of two classes, epithelium and other, derived
from immunofluorescence restaining. We retained the official training split and merged the official validation and test sets into a single test set, giving 21,912 tiles in the full training set and 4,597 tiles in the test set.

\paragraph{CAMELYON16}~\citep{ehteshami2017diagnostic} contains H\&E-stained WSIs of sentinel lymph node sections from breast cancer patients at Radboud University Medical Center and University Medical Center Utrecht, digitized at $40\times$ magnification. We constructed our dataset from the 158 WSIs containing metastasis, each of which carries pixel-level annotations of three classes: background, normal tissue and metastatic tissue. We partitioned these at the slide level at a ratio of 80:20, then extracted non-overlapping tiles of $512 \times 512$ pixels at $20\times$ magnification, retaining only those with more than 50\% non-background area. This yielded 365,112 tiles from 126 WSIs in the full training set and 114,660 tiles from 32 WSIs in the test set.

\newpage
\section{Implementation details}

\subsection{Hyperparameter settings of baselines}
\label{app:baseline_hparams}

\paragraph{Why natural-image hyperparameters are not used directly.}
All baselines except Random and K-center expose hyperparameters whose default values or robust ranges were determined on natural-image benchmarks. As illustrated in Section~\ref{sec:intro}, neither transfers reliably to histopathology. First, a range reported as robust on natural images can produce substantial variation on histopathology data: InfoMax~\citep{tan2025data} was reported to be robust to $\alpha \in [0.3, 4]$ on CC12M, yet on PANDA at a 2\% retention rate, $\alpha = 0.3, 1, 2, 3, 4$ yields QWK values of 0.60, 0.58, 0.70, 0.66 and 0.69, respectively. Second, the prescribed default can be among the worst choices: ELFS~\citep{zheng2025elfs} sets the number of pseudo-classes $\tilde{C}$ to the number of classes $C$, but on TCGA-BRCA ($C = 2$) at a 5\% retention rate, $\tilde{C} = 2, 4, 10, 15, 20$ yields balanced accuracies of 0.71, 0.85, 0.84, 0.83 and 0.83, with the default performing worst. Third, the default may be undefined: for segmentation on BEETLE, tiles do not carry a single class label, and $\tilde{C} = 4, 10, 15, 20, 25$ at a 0.02\% retention rate yields test weighted Dice values of 0.81, 0.84, 0.70, 0.84 and 0.78, with no monotonic trend.

\paragraph{Sweep protocol.}
Tuning these hyperparameters for each cohort would require additional labels, which conflicts with the purpose of coreset selection; we therefore assume that no labels are available for tuning. For each baseline, rather than selecting the best setting, we sweep a characteristic hyperparameter over five values and report the mean test performance across the five resulting coresets, together with its SD. The five values are sampled within the range reported to be robust on natural-image benchmarks, where such a range is given, and otherwise around the reported default value, including the default itself. All other hyperparameters are kept at their default settings. For Random, K-center, and ZCore, we sweep the random seed. The mean thus indicates the expected performance of a baseline on a new cohort without labels for tuning, while the SD of the test performance across the five settings indicates the sensitivity to the swept hyperparameter. GCcore is deterministic and hyperparameter-free and therefore has no SD.

\paragraph{Per-method settings.}
Random, K-center~\citep{sener2017active} and ZCore~\citep{griffin2026zero} are swept over the random seeds $\{17, 83, 146, 271, 394\}$. 
K-means, Proto-Hard, Proto-Easy, Moderate and ELFS-Proto are swept over the number of clusters $M$, and ELFS-Pseudo over the number of pseudo-classes $\tilde{C}$. Following the original works, the default is the number of classes $C$ of the dataset, so the five values are dataset-specific. The values for each dataset are listed in Table~\ref{tab:hparam_values}.
CCS~\citep{zheng2022coverage} is swept over the hard-cutoff rate $\beta$. The original work selects $\beta$ by grid search for each dataset and pruning rate, which requires labels; the reported optimal values lie between 0 and 0.5 and increase with the pruning rate, ranging from 0.2 to 0.5 at pruning rates of 90\% and above. Since a majority of our experiments operate at pruning rates above 90\%, we use $\beta \in \{0.1, 0.2, 0.3, 0.4, 0.5\}$.
D\textsuperscript{2}-Uniform and D\textsuperscript{2}-Proto~\citep{maharana2024mathbb} are swept over the number of nearest neighbors $k_{\mathrm{nn}} \in \{1, 5, 10, 15, 20\}$.
InfoMax-Hard and InfoMax-Easy~\citep{tan2025data} are swept over $\alpha \in \{0.3, 1, 2, 3, 4\}$, sampled within the range reported as robust.

\begin{table}[h]
\centering
\caption{Swept hyperparameter and its five values for each baseline. (a)~Hyperparameters swept for each baseline. (b)~Dataset-specific values of the number of clusters $M$ and the number of pseudo-classes $\tilde{C}$.}
\label{tab:hparam_values}
\footnotesize
\renewcommand{\arraystretch}{1.15}

\textbf{(a)} Swept hyperparameters\\[2pt]
\begin{tabularx}{0.8\linewidth}{@{}>{\raggedright\arraybackslash}X >{\raggedright\arraybackslash}p{3.4cm} >{\raggedright\arraybackslash}p{3.6cm}@{}}
\toprule
\textbf{Baseline} & \textbf{Swept hyperparameter} & \textbf{Values} \\
\midrule
Random, K-center, ZCore & Random seed $s$ & $\{17, 83, 146, 271, 394\}$ \\
K-means, Proto-Hard, Proto-Easy, Moderate, ELFS-Proto & Number of clusters $M$ & Dataset-specific, see (b) \\
ELFS-Pseudo & Number of pseudo-classes $\tilde{C}$ & Dataset-specific, see (b) \\
CCS & Hard-cutoff rate $\beta$ & $\{0.1, 0.2, 0.3, 0.4, 0.5\}$ \\
D\textsuperscript{2}-Uniform, D\textsuperscript{2}-Proto & Number of nearest neighbors $k_{\mathrm{nn}}$ & $\{1, 5, 10, 15, 20\}$ \\
InfoMax-Hard, InfoMax-Easy & Pairwise weight $\alpha$ & $\{0.3, 1, 2, 3, 4\}$ \\
\bottomrule
\end{tabularx}

\vspace{8pt}
\textbf{(b)} Dataset-specific values of $M$ and $\tilde{C}$\\[2pt]
\begin{tabularx}{0.8\linewidth}{@{}>{\raggedright\arraybackslash}p{2.6cm} >{\raggedright\arraybackslash}X c >{\raggedright\arraybackslash}p{3.6cm}@{}}
\toprule
\textbf{Task} & \textbf{Dataset} & $\boldsymbol{C}$ & \textbf{Values of $M$ and $\tilde{C}$} \\
\midrule
\multirow{3}{*}{WSI classification}
 & PANDA      & 6 & $\{6, 10, 15, 20, 25\}$ \\
 & IMP-CRS    & 3 & $\{3, 4, 10, 15, 20\}$ \\
 & TCGA-BRCA  & 2 & $\{2, 4, 10, 15, 20\}$ \\
\midrule
\multirow{4}{*}{Tile classification}
 & WSSS4LUAD  & 3 & $\{3, 4, 10, 15, 20\}$ \\
 & SICAP      & 4 & $\{4, 10, 15, 20, 25\}$ \\
 & CRC-100k   & 9 & $\{9, 10, 15, 20, 25\}$ \\
 & DeepHP     & 2 & $\{2, 4, 10, 15, 20\}$ \\
\midrule
\multirow{3}{*}{Segmentation}
 & BEETLE     & 4 & $\{4, 10, 15, 20, 25\}$ \\
 & panCK      & 2 & $\{2, 4, 10, 15, 20\}$ \\
 & CAMELYON16 & 3 & $\{3, 4, 10, 15, 20\}$ \\
\bottomrule
\end{tabularx}
\end{table}


\subsection{Software and hardware specifications}
We implemented our method in PyTorch~\citep{paszke2017automatic}; experiments were conducted across two servers, one with two NVIDIA RTX~A6000 GPUs (48~GB each) and one with three NVIDIA RTX~PRO~6000 Blackwell GPUs (96~GB each). FM embeddings are extracted once and cached, so every selection run and every downstream head reuses the same features.

For all model training, only the head parameters are optimized, while the encoder remains frozen throughout.
All heads are trained using AdamW with a learning rate of $10^{-3}$ and weight decay of $10^{-4}$, with batch sizes of 256 for classification heads and 32 for segmentation heads. When the subset size $k$ is smaller than the batch size, minibatches are sampled from the subset with replacement, ensuring that each gradient step uses a full batch.
Following~\citep{zheng2022coverage, maharana2024mathbb, tan2025data}, the number of epochs is set per run so that the number of iterations on the selected subset matches that of the corresponding full-set run. Identical settings are used across all servers, selection strategies and retention rates. To mitigate the influence of the training procedure on test performance, the task-specific head is trained five times on each coreset, whether selected by GCcore or by a single configuration (hyperparameter value) of a baseline, using five different training seeds that govern head initialization and minibatch ordering. The mean test performance across these five runs is reported.

\subsection{Segmentation head architecture}
\label{seg_head}
The segmentation head operates on the spatial token grid of the frozen encoder and is the only trainable component of the model. To obtain this grid, embeddings are extracted from the final transformer block of the frozen tile encoder, and the prefix tokens ([CLS] and any register tokens) are discarded to retain only the patch tokens carrying spatial information. The retained tokens are then reshaped from a flattened sequence into a two-dimensional grid, preserving spatial correspondence with the original image locations, so that each token encodes features from a distinct non-overlapping region of the input. Throughout, let $D$ denote the encoder embedding dimensionality, $g \times g$ the token grid, and $R$ the working resolution at which the loss and the evaluation metrics are computed; the input to the head is therefore a feature map of size $D \times g \times g$.

The head begins with a $1\times1$ convolutional projection that preserves the channel width at $D$, followed by batch normalization and a GELU activation. Four successive transposed-convolutional blocks (kernel size 4, stride 2, padding 1) then upsample the spatial resolution by a factor of two at each stage while halving the channel dimensionality, each block again followed by batch normalization and GELU. A final $1\times1$ convolution projects the resulting feature maps onto the $C$ segmentation classes, yielding pixel-wise logits at $16\times$ the token-grid resolution. For H-optimus-1, a ViT-g/14 at an input resolution of $224 \times 224$, the 257 output tokens of the final block reduce to 256 patch tokens once the [CLS] token is discarded, giving $g = 16$ and $D = 1536$; the decoder therefore takes the feature map from $16\times16$ to $32\times32$, $64\times64$, $128\times128$ and finally $256\times256$, with the channel width reducing from 1536 to 768, 384, 192 and 96.

Because both the token grid size $g$ and the working resolution $R$ may vary across encoders and datasets, the logits are bilinearly resampled from $16g \times 16g$ to $R \times R$; this is the identity when $16g = R$ and a mild resampling otherwise. We set $R = 224$ for BEETLE, panCK and CAMELYON16, so the $256\times256$ logit map is slightly reduced in each case. Masks are resized to $R \times R$ by nearest-neighbor interpolation, and patches are resized to the encoder's native input resolution without cropping, so that image and label remain in spatial correspondence.

\newpage
\section{Further results}
\label{further_results}

\subsection{GCcore is deterministic}
\label{GC_deter}
Fig.~\ref{fig:combined_box} shows that the performance of the 14 baseline methods is sensitive to their respective characteristic hyperparameters. Varying each hyperparameter across five values leads to a SD of up to 0.507 in test performance (Proto-Easy on IMP-CRS at 10\% retention). Across the 40 dataset--retention-rate combinations (10 datasets $\times$ four low retention rates), Proto-Easy is the most sensitive method, with a mean range of 0.231 (SD $=$ 0.099), while the least sensitive, Random, still has a mean range of 0.051 (SD $=$ 0.021).

GCcore exceeds the mean of every baseline in 26 of the 40 combinations. Among the remaining 14 combinations, for 73\% of the 55 baseline configurations whose mean performance exceeds that of GCcore, the advantage rests on only one or two settings. For example, at 0.5\% retention on PANDA, K-means achieves QWK values of 0.723, 0.717, 0.738, 0.792, and 0.713 for 6, 10, 15, 20, and 25 clusters, respectively. The resulting mean of 0.737 is higher than the GCcore value of 0.733; however, K-means exceeds GCcore at only two of the five settings. Similarly, at 0.5\% retention on SICAP, CCS achieves 0.580, 0.713, 0.626, 0.762, and 0.677 across its five settings. Although its mean of 0.672 is 0.004 higher than the GCcore value of 0.668, the range across settings is 0.182, approximately 48 times this difference.

Importantly, identifying the most favorable hyperparameter setting requires evaluating the candidate coresets on labeled data, the need for which coreset selection is intended to reduce. Consequently, the advantage of a baseline at its favorable settings cannot be reliably obtained in a practical, label-efficient scenario. In contrast, GCcore requires no hyperparameter tuning and achieves higher overall performance, making it directly applicable in practice.

\subsection{GCcore is robust to the choice of feature extractor}
\label{ablation_FMchoice}
Figure~\ref{fig:ablation_FMchoice_PANDA} and Figure~\ref{fig:ablation_FMchoice_BEETLE}
disentangle two factors that are usually confounded in coreset evaluations: how well a selection method performs, and how much of that performance is inherited from the feature extractor on which it is built. The full-dataset panels establish the ceiling of the three slide encoders (PRISM~\citep{shaikovski2024prism} 0.856 $>$ TITAN~\citep{ding2025multimodal} 0.833 $>$ MADELEINE~\citep{jaume2024multistain} 0.748) and
the three tile encoders (H-optimus-1~\citep{scalbert2026hoptimus1} 0.923 $\approx$ Virchow2~\citep{zimmermann2024virchow2} 0.922 $>$ Prov-GigaPath~\citep{xu2024whole} 0.909).


\textbf{GCcore is both the strongest method and the one least dependent on the slide encoder (Figure~\ref{fig:ablation_FMchoice_PANDA}).} Over the 12 combinations of encoder and retention rate (3 encoders $\times$ 4 low retention rates, 0.5--5\%), GCcore ranks in the top four of the 15 selection strategies in all 12 and is the single best method in 8 of them (TITAN at every rate, MADELEINE at 1\%, 2\% and 5\%, and PRISM at 2\%). No baseline is comparably stable: the next most consistent method, ELFS-Pseudo, falls to 7th, and every other baseline falls to 8th or worse in at least one combination. We quantify encoder sensitivity relative to each encoder's own ceiling: for a given method and retention rate we compute how far the method falls below the full-dataset score of each encoder, take the range of these three gaps, and average it over the four low retention rates. Smaller values indicate a method whose loss relative to the achievable ceiling is more stable across encoders. GCcore attains the smallest value of all methods (11.7 points), ahead of ELFS-Proto (13.3) and far ahead of D\textsuperscript{2}-Uniform (32.6); the full ordering is GCcore (11.7) $<$ ELFS-Proto (13.3) $<$ Random (13.9) $<$ Moderate (14.2) $<$ CCS (15.4) $<$ InfoMax-Easy (16.3) $<$ K-means (17.3) $<$ D\textsuperscript{2}-Proto (18.4) $<$ ELFS-Pseudo (18.8) $<$ InfoMax-Hard (24.5) $<$ K-center (25.6) $<$ ZCore (26.4) $<$ Proto-Hard (28.4) $<$ Proto-Easy (30.2) $<$ D\textsuperscript{2}-Uniform (32.6).

\textbf{GCcore remains robust across tile encoders (Figure~\ref{fig:ablation_FMchoice_BEETLE}).} Over the 12 combinations of tile encoder and retention rate (3 encoders $\times$ 4 low retention rates, 0.02--0.2\%), GCcore is the best of the 15 selection strategies in 10 and is never worse than 2nd, giving it both the best worst-case rank (2nd) and the best mean rank (1.17) of any method. It also has the lowest encoder sensitivity (1.5 points), ahead of CCS (1.8) and ELFS-Proto (2.0).

\begin{figure}[t]
  \centering
  \includegraphics[width=\linewidth]{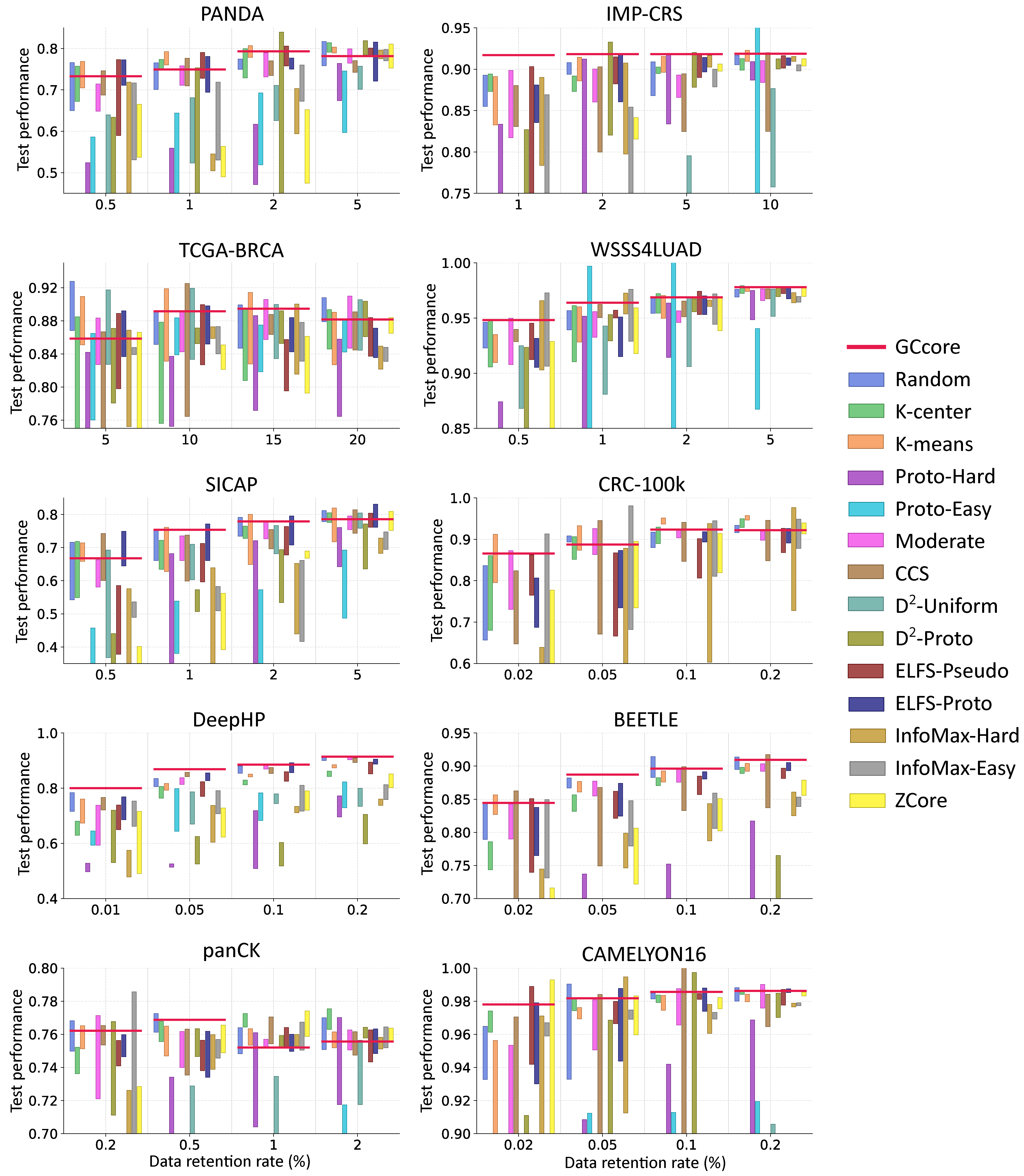}
  \caption{Baselines are sensitive to their hyperparameters; GCcore is deterministic and hyperparameter-free. Test performance at the four low data retention rates for 10 histopathology datasets spanning WSI classification (PANDA, IMP-CRS, TCGA-BRCA), tile classification (WSSS4LUAD, SICAP, CRC-100k, DeepHP) and segmentation (BEETLE, panCK, CAMELYON16). Each shaded box spans the mean ± one SD of a baseline across its five hyperparameter values; the box therefore reports the spread induced by the selection procedure itself. GCcore (red) is deterministic and yields a single test performance per retention rate, shown as a horizontal level over the corresponding group.}
  \label{fig:combined_box}
\end{figure}

Taken together, GCcore attains the top mean rank across six different feature extractors (Table~\ref{tab:rank_sd_fm}) and its advantage is a property of the selection method itself rather than of any particular encoder.

\begin{table}[t]
\centering
\caption{%
Mean rank of the 15 selection methods under each FM as the feature extractor; the subscript is the SD of that rank across five hyperparameter values. GCcore is deterministic and reports no SD. \textbf{Mean}: average across the six feature extractors.}
\label{tab:rank_sd_fm}
\small
\setlength{\tabcolsep}{0pt}
\begin{tabular*}{0.8\linewidth}{@{\extracolsep{\fill}} l ccc ccc c @{}}
\toprule
\multirow[c]{3}{*}{\makecell[l]{Selection\\ method}}
& \multicolumn{3}{c}{Slide encoder (PANDA)}
& \multicolumn{3}{c}{Tile encoder (BEETLE)}
& \multirow[c]{3}{*}{\textbf{Mean}} \\
\cmidrule(lr){2-4} \cmidrule(lr){5-7}
& PRISM & TITAN & \makecell{MADE-\\LEINE}
& \makecell{H-opti-\\mus-1} & \makecell{Prov-\\GigaPath} & \makecell{Vir-\\chow2} & \\
\midrule
GCcore       & $3.0$ & $\underline{\mathbf{1.0}}$ & $\underline{\mathbf{1.5}}$ & $\underline{\mathbf{1.3}}$ & $\underline{\mathbf{1.3}}$ & $\underline{\mathbf{1.0}}$ & $\underline{\mathbf{1.5}}$ \\
Random       & \ms{6.5}{1.8} & \ms{3.8}{0.4} & \ms{3.8}{0.7} & \ms{2.3}{1.6} & \ms{4.0}{1.6} & \ms{3.3}{1.9} & \ms{3.9}{1.3} \\
Moderate     & \ms{7.8}{2.8} & \ms{4.0}{1.8} & \ms{4.0}{1.5} & \ms{3.8}{1.2} & \ms{2.5}{1.1} & \ms{3.0}{0.7} & \ms{4.2}{1.5} \\
K-means      & \msb{2.0}{0.8} & \ms{6.8}{2.3} & \ms{4.8}{2.5} & \ms{3.5}{0.8} & \ms{4.5}{1.1} & \ms{4.0}{1.0} & \ms{4.3}{1.4} \\
ELFS-Proto   & \ms{5.8}{1.3} & \ms{3.8}{0.7} & \ms{2.0}{1.2} & \ms{4.3}{1.2} & \ms{6.0}{1.6} & \ms{4.8}{1.2} & \ms{4.4}{1.2} \\
ELFS-Pseudo  & \ms{4.5}{1.7} & \ms{5.3}{1.2} & \ms{6.8}{3.1} & \ms{6.8}{0.9} & \ms{5.5}{1.6} & \ms{5.3}{1.3} & \ms{5.7}{1.6} \\
K-center     & \ms{3.0}{1.5} & \ms{10.5}{0.8} & \ms{9.3}{0.3} & \ms{6.5}{0.4} & \ms{4.5}{1.5} & \ms{7.3}{1.1} & \ms{6.8}{0.9} \\
CCS          & \ms{7.3}{0.9} & \ms{3.8}{3.0} & \ms{5.5}{2.4} & \ms{8.5}{3.3} & \ms{9.5}{2.3} & \ms{8.0}{1.7} & \ms{7.1}{2.3} \\
InfoMax-Easy & \ms{8.8}{3.3} & \ms{8.5}{1.8} & \ms{9.0}{1.8} & \ms{8.5}{1.5} & \ms{9.0}{0.8} & \ms{9.0}{1.3} & \ms{8.8}{1.8} \\
ZCore        & \ms{11.5}{1.3} & \ms{9.8}{0.8} & \ms{12.3}{0.7} & \ms{10.3}{0.4} & \ms{8.8}{1.1} & \ms{10.0}{0.8} & \ms{10.4}{0.8} \\
InfoMax-Hard & \ms{11.0}{1.4} & \ms{11.5}{1.1} & \ms{11.5}{0.8} & \ms{10.5}{0.5} & \ms{10.5}{0.6} & \ms{10.5}{1.4} & \ms{10.9}{1.0} \\
D\textsuperscript{2}-Proto     & \ms{9.3}{2.7} & \ms{10.3}{2.5} & \ms{8.5}{2.2} & \ms{13.5}{0.3} & \ms{13.5}{0.1} & \ms{13.4}{0.3} & \ms{11.4}{1.4} \\
Proto-Hard   & \ms{14.8}{1.6} & \ms{13.5}{1.0} & \ms{13.0}{1.0} & \ms{12.0}{0.5} & \ms{12.0}{0.8} & \ms{12.5}{1.0} & \ms{13.0}{1.0} \\
D\textsuperscript{2}-Uniform   & \ms{11.0}{0.9} & \ms{13.5}{2.9} & \ms{13.8}{1.3} & \ms{14.3}{0.4} & \ms{14.0}{0.0} & \ms{14.6}{0.6} & \ms{13.5}{1.0} \\
Proto-Easy   & \ms{14.0}{2.1} & \ms{14.3}{1.6} & \ms{14.5}{0.7} & \ms{14.3}{0.1} & \ms{14.5}{0.1} & \ms{13.5}{0.8} & \ms{14.2}{0.9} \\
\bottomrule
\end{tabular*}
\end{table}

\begin{figure}[t]
  \centering
  \includegraphics[width=\linewidth]{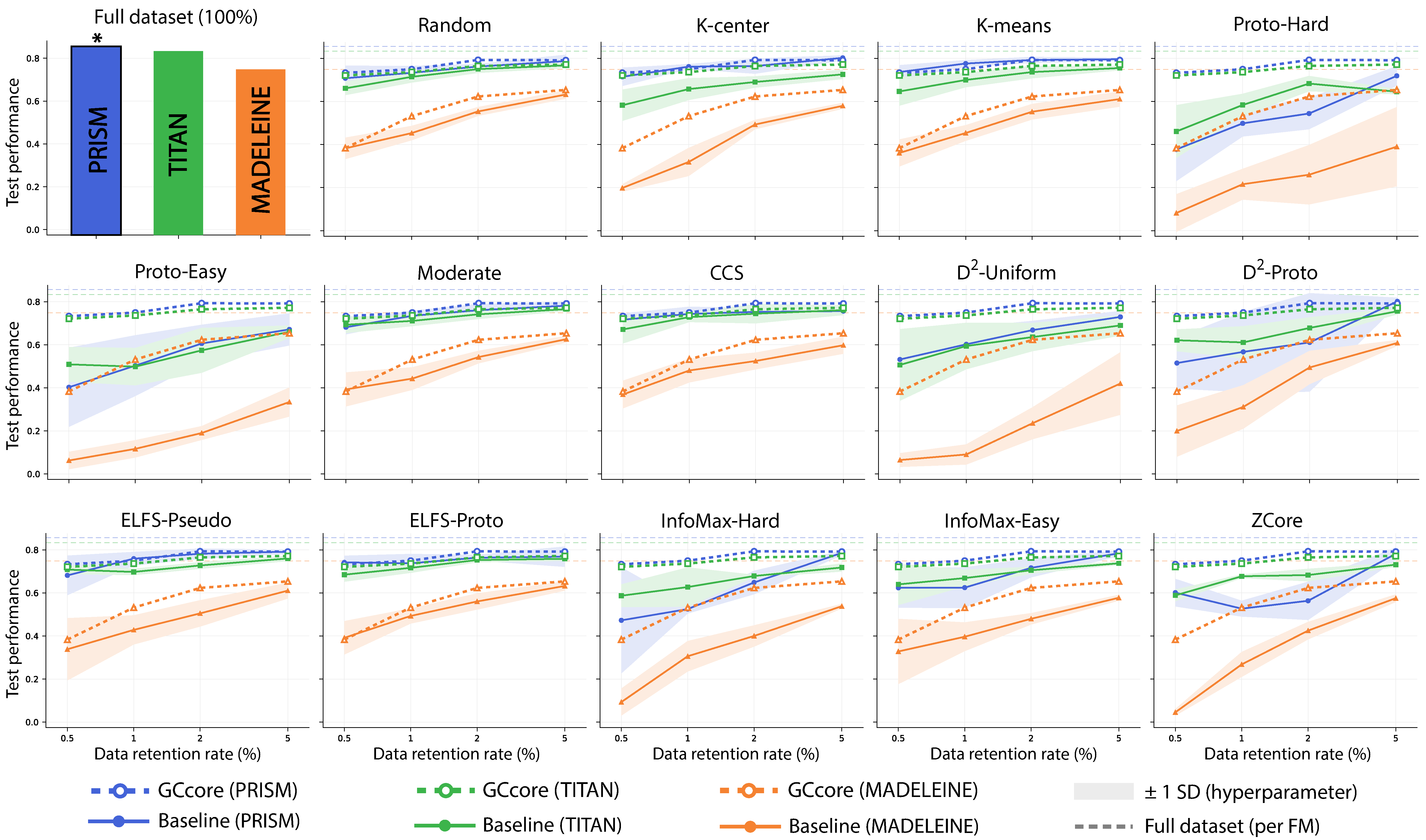}
  \caption{Test performance on different data retention rates (0.5–5\%, log scale) for three slide encoders (FMs): PRISM (blue), TITAN (green), and MADELEINE (orange) on PANDA dataset. The leftmost panel reports full-dataset (100\%) test performance per FM, with the strongest model starred. The remaining 14 panels each compare GCcore (dotted) against one baseline selection method (solid).}
  \label{fig:ablation_FMchoice_PANDA}
\end{figure}

\begin{figure}[t]
  \centering
  \includegraphics[width=\linewidth]{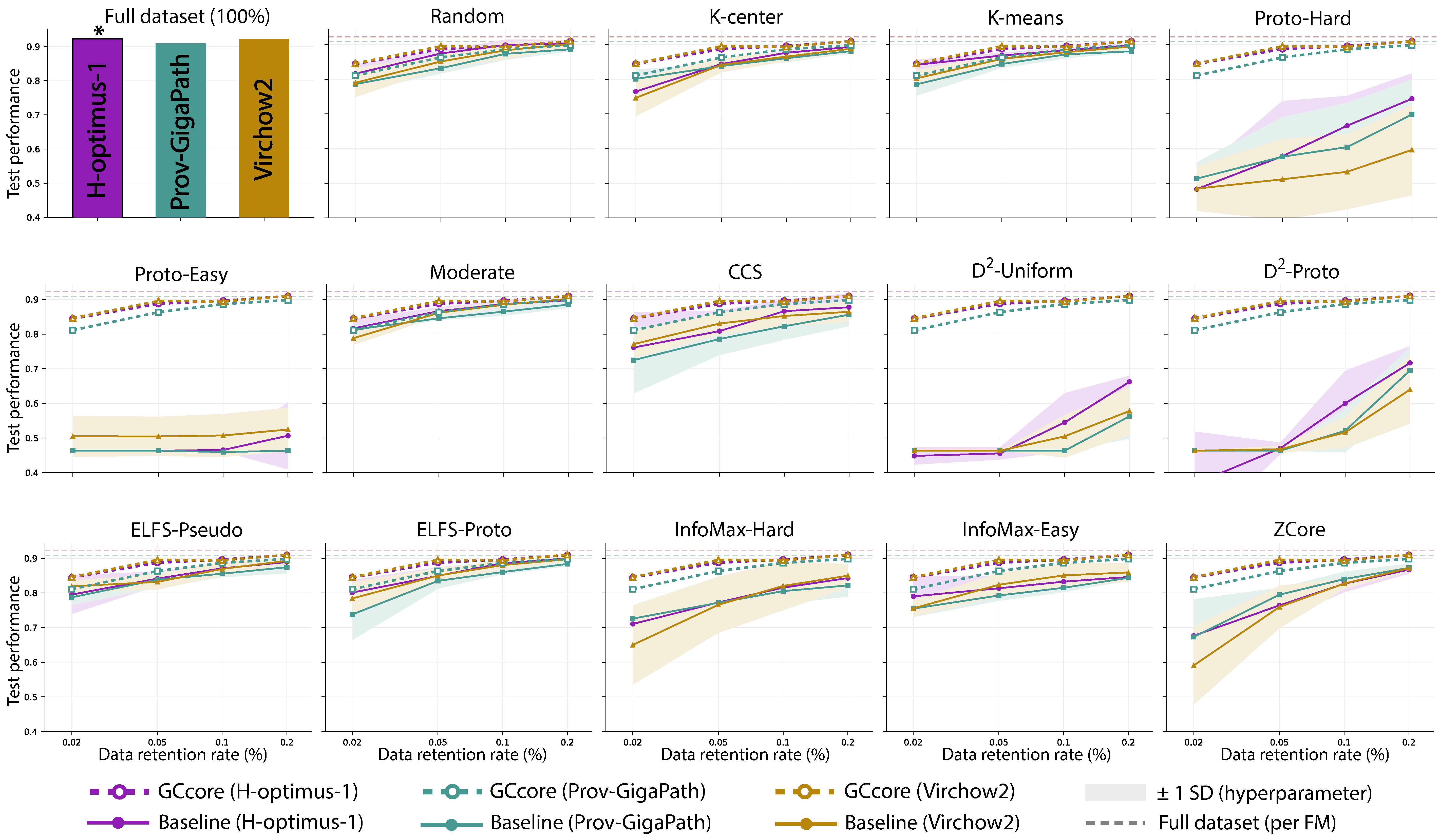}
  \caption{Test performance on different data retention rates (0.02–0.2\%, log scale) for three tile encoders (FMs): H-optimus-1 (purple), Prov-GigaPath (teal), and Virchow2 (ochre) on BEETLE dataset. Panels and line styles as in Figure~\ref{fig:ablation_FMchoice_PANDA}.}
  \label{fig:ablation_FMchoice_BEETLE}
\end{figure}

\begin{table}[t]
\centering
\caption{%
Mean rank of the 15 selection methods for every pair of selection and training FMs on PANDA; the subscript is the SD of that rank across five hyperparameter values. GCcore is deterministic and reports no SD. Column groups give the FM whose embedding was used for coreset selection; the three sub-columns within a group give the frozen FM used for model training and testing. \textbf{Mean}: average across the nine selection–training FM pairs.}
\label{tab:rank_select_encoder}
\small
\setlength{\tabcolsep}{0pt}
\begin{tabular*}{\linewidth}{@{\extracolsep{\fill}} l ccc ccc ccc c @{}}
\toprule
\multirow[c]{2}{*}{\makecell[l]{Selection\\ method}}
& \multicolumn{3}{c}{Selection: PRISM}
& \multicolumn{3}{c}{Selection: TITAN}
& \multicolumn{3}{c}{Selection: MADELEINE}
& \multirow[c]{2}{*}{\textbf{Mean}} \\
\cmidrule(lr){2-4} \cmidrule(lr){5-7} \cmidrule(lr){8-10}
& PRISM & TITAN & \makecell{MADE-\\LEINE} & PRISM & TITAN & \makecell{MADE-\\LEINE} & PRISM & TITAN & \makecell{MADE-\\LEINE} & \\
\midrule
GCcore       & $3.0$ & $\underline{\mathbf{3.3}}$ & $\underline{\mathbf{1.8}}$ & $\underline{\mathbf{1.3}}$ & $\underline{\mathbf{1.0}}$ & $2.0$ & $4.8$ & $5.8$ & $\underline{\mathbf{1.5}}$ & $\underline{\mathbf{2.7}}$ \\
ELFS-Pseudo  & \ms{4.5}{1.7} & \ms{3.8}{1.6} & \ms{3.3}{1.2} & \ms{4.0}{1.2} & \ms{5.3}{1.2} & \msb{1.5}{0.7} & \ms{3.5}{1.2} & \ms{4.3}{1.3} & \ms{6.8}{3.1} & \ms{4.1}{1.5} \\
K-means      & \msb{2.0}{0.8} & \ms{4.3}{1.9} & \ms{3.8}{1.6} & \ms{3.5}{1.8} & \ms{6.8}{2.3} & \ms{7.3}{2.6} & \ms{3.8}{0.5} & \ms{4.8}{1.0} & \ms{4.8}{2.5} & \ms{4.5}{1.7} \\
Moderate     & \ms{7.8}{2.8} & \ms{3.8}{0.4} & \ms{4.5}{1.4} & \ms{6.8}{2.5} & \ms{4.0}{1.8} & \ms{3.3}{2.2} & \ms{6.3}{2.5} & \msb{4.0}{0.5} & \ms{4.0}{1.5} & \ms{4.9}{1.7} \\
Random       & \ms{6.5}{1.8} & \ms{4.0}{0.8} & \ms{6.8}{1.1} & \ms{6.3}{1.8} & \ms{3.8}{0.4} & \ms{5.5}{0.8} & \ms{4.0}{1.8} & \msb{4.0}{0.9} & \ms{3.8}{0.7} & \ms{4.9}{1.1} \\
ELFS-Proto   & \ms{5.8}{1.3} & \ms{6.5}{2.6} & \ms{7.3}{1.8} & \ms{7.3}{0.7} & \ms{3.8}{0.7} & \ms{6.0}{1.4} & \msb{3.3}{2.3} & \msb{4.0}{2.5} & \ms{2.0}{1.2} & \ms{5.1}{1.6} \\
CCS          & \ms{7.3}{0.9} & \ms{9.0}{0.9} & \ms{10.3}{1.5} & \ms{8.3}{1.1} & \ms{3.8}{3.0} & \ms{4.8}{1.9} & \ms{5.0}{2.7} & \ms{5.0}{2.2} & \ms{5.5}{2.4} & \ms{6.5}{1.8} \\
InfoMax-Easy & \ms{8.8}{3.3} & \ms{8.8}{1.4} & \ms{3.8}{2.6} & \ms{5.5}{0.6} & \ms{8.5}{1.8} & \ms{9.0}{2.4} & \ms{8.3}{2.0} & \ms{7.5}{1.8} & \ms{9.0}{1.8} & \ms{7.7}{2.0} \\
K-center     & \ms{3.0}{1.5} & \ms{5.8}{2.9} & \ms{10.3}{1.3} & \ms{9.8}{1.1} & \ms{10.5}{0.8} & \ms{8.8}{1.5} & \ms{11.0}{2.5} & \ms{9.3}{0.6} & \ms{9.3}{0.3} & \ms{8.6}{1.4} \\
D\textsuperscript{2}-Proto     & \ms{9.3}{2.7} & \ms{12.0}{0.5} & \ms{9.8}{1.9} & \ms{10.9}{3.5} & \ms{10.3}{2.5} & \ms{9.3}{2.0} & \ms{8.8}{2.2} & \ms{8.0}{2.7} & \ms{8.5}{2.2} & \ms{9.6}{2.3} \\
InfoMax-Hard & \ms{11.0}{1.4} & \ms{11.5}{1.1} & \ms{6.8}{2.5} & \ms{7.8}{2.2} & \ms{11.5}{1.1} & \ms{11.0}{0.8} & \ms{12.3}{0.8} & \ms{10.3}{0.5} & \ms{11.5}{0.8} & \ms{10.4}{1.2} \\
ZCore        & \ms{11.5}{1.3} & \ms{8.3}{1.9} & \ms{10.5}{1.3} & \ms{8.9}{1.0} & \ms{9.8}{0.8} & \ms{9.8}{1.2} & \ms{13.0}{0.5} & \ms{12.8}{0.7} & \ms{12.3}{0.7} & \ms{10.7}{1.0} \\
Proto-Hard   & \ms{14.8}{1.6} & \ms{10.8}{1.4} & \ms{13.0}{1.8} & \ms{13.3}{0.6} & \ms{13.5}{1.0} & \ms{13.3}{0.2} & \ms{11.0}{3.8} & \ms{12.3}{4.3} & \ms{13.0}{1.0} & \ms{12.8}{1.7} \\
D\textsuperscript{2}-Uniform   & \ms{11.0}{0.9} & \ms{13.5}{1.5} & \ms{13.5}{0.6} & \ms{13.0}{2.7} & \ms{13.5}{2.9} & \ms{13.8}{1.6} & \ms{13.5}{0.5} & \ms{14.0}{2.1} & \ms{13.8}{1.3} & \ms{13.3}{1.6} \\
Proto-Easy   & \ms{14.0}{2.1} & \ms{15.0}{0.3} & \ms{15.0}{0.1} & \ms{13.8}{1.8} & \ms{14.3}{1.6} & \ms{15.0}{0.2} & \ms{11.8}{0.9} & \ms{14.3}{3.2} & \ms{14.5}{0.7} & \ms{14.2}{1.2} \\
\bottomrule
\end{tabular*}
\end{table}

\subsection{GCcore transfers across selection and training encoders}

As explained in Section~\ref{exp}, the embeddings extracted from a single FM were used for both coreset selection and downstream model training. In this section we investigate how far GCcore transfers across the FM used for coreset selection and the FM used for downstream training, using the PANDA dataset.

Table~\ref{tab:rank_select_encoder} shows that GCcore attains the best mean rank of all methods (2.69). The transfer is strongest when a stronger FM is used for selection. Over the six pairs in which PRISM or TITAN provides the selection embedding, GCcore attains a mean rank of 2.04 and is never worse than 3.25, irrespective of which of the three FMs is used for training; in the two pairs where it is not first, it is second or joint second. When MADELEINE provides the selection embedding, however, GCcore remains best in the matched setting (1.50) but falls to 4.75 (5th) and 5.75 (7th) when the coreset is transferred to PRISM and TITAN, giving a mean rank of 4.00 over those three pairs. MADELEINE also yields the lowest full-dataset ceiling of the three FMs (QWK 0.748, against 0.856 for PRISM and 0.833 for TITAN). 

These results indicate that, on PANDA, a coreset selected by GCcore with a stronger encoder (PRISM or TITAN) remains among the top-ranked choices for training with any of the three encoders, so it need not be recomputed for a different downstream feature extractor. Selecting with the weakest available encoder, on the other hand, is not advisable when the coreset is intended for transfer.

\subsection{GCcore is robust to class imbalance}
\label{sec:classimb}

\textbf{Handling class imbalance during training.}
As described in Section~\ref{exp}, the classification and segmentation heads were trained with a class-weighted cross-entropy loss on all datasets except WSSS4LUAD, in order to mitigate the class imbalance of the training set. Because WSSS4LUAD is a multi-label problem, a weighted categorical cross-entropy loss is not appropriate as the three tissue classes are not mutually exclusive and a single tile may contain several tissue types. We therefore weight the positive term of class $c$ by
\begin{equation}
  \lambda_{c} = \min\!\left(\frac{n^{-}_{c}}{n^{+}_{c}},\, \lambda_{\max}\right),
  \qquad \lambda_{\max} = 10 ,
\end{equation}
where $n^{+}_{c}$ and $n^{-}_{c}$ are the positive and negative counts of class $c$ within the selected coreset, since that is the distribution to which the head is exposed. The cap $\lambda_{\max}$ prevents a class that becomes very rare in a small coreset from dominating the gradient. 

\textbf{Simulating strongly imbalanced training sets.}
To further investigate the robustness of GCcore to strongly class-imbalanced training data, we simulated imbalanced training sets from the relatively balanced PANDA training set (grade 0, 2,313; grade 1, 2,133; grade 2, 1,074; grade 3, 994; grade 4, 999; grade 5, 979 WSIs; 8,492 in total). In each simulation, five of the six ISUP grades were drawn at random and down-sampled to 1\%, 20\%, 40\%, 60\% and 80\% of their WSIs, respectively, while the remaining grade was retained in full; within each down-sampled grade, the retained WSIs were selected uniformly at random. Repeating this procedure three times yielded three imbalanced training sets (IMBs). In IMB1, grades 5, 3, 1, 4 and 2 retained 1\%, 20\%, 40\%, 60\% and 80\% of their WSIs, respectively, and grade 0 was retained in full, giving 4,833 WSIs in total (grade 0, 2,313; grade 1, 853; grade 2, 859; grade 3, 199; grade 4, 599; grade 5, 10); the imbalance ratio between the largest and the smallest grade thereby increased from 2.4 to 231.3. In IMB2, grades 0, 5, 3, 1 and 2 retained 1\%, 20\%, 40\%, 60\% and 80\% of their WSIs, respectively, and grade 4 was retained in full, giving 3,755 WSIs in total (grade 0, 23; grade 1, 1,280; grade 2, 859; grade 3, 398; grade 4, 999; grade 5, 196; imbalance ratio 55.7). In IMB3, grades 4, 0, 2, 1 and 3 retained 1\%, 20\%, 40\%, 60\% and 80\% of their WSIs, respectively, and grade 5 was retained in full, giving 3,957 WSIs in total (grade 0, 463; grade 1, 1,280; grade 2, 430; grade 3, 795; grade 4, 10; grade 5, 979; imbalance ratio 128.0). 

Each IMB was then treated as the full dataset from which coresets were selected, at the four low retention rates $\{0.5, 1, 2, 5\}\%$ and using each of the 15 selection strategies. Performance is always reported on the complete test set (2,124 WSIs).

\textbf{GCcore ranks best overall across three imbalance regimes.}
Table~\ref{tab:imb_rank_sd} shows that GCcore attains the best mean rank over the three IMBs (3.3 of 15) and is the only strategy placed among the two best on all three sets. Notably, GCcore reaches a QWK of 0.657 from only 19 WSIs (0.5\% retention) on IMB2, recovering 90.7\% of the QWK obtained with the entire imbalanced training set (0.724), against 0.549 for the best competing strategy and 0.477 for Random. No competing strategy reaches this level at retention rates up to 2\% (75 WSIs), and Random requires 20\% retention (751 WSIs) to reach a comparable level; that is, a coreset of 19 slides selected by GCcore performs on par with a randomly selected coreset roughly 40 times larger.

\textbf{The advantage of GCcore does not necessarily arise from class re-balancing.}
A natural explanation for robustness under imbalance is that a selection strategy implicitly re-balances the label distribution. We tested this by quantifying the class balance of each coreset as the normalized Shannon entropy over the six ISUP grades (evenness; 1 = uniform) and relating it to test performance (Fig.~\ref{fig:ablation_classimb_fig2}).
Balance and performance are positively associated across strategies on IMB1 ($R^2 = 0.57$, $p < 0.001$) and on IMB3 ($R^2 = 0.42$, $p < 0.001$), and on both GCcore selects the most balanced coreset of all 15 strategies at 0.5 and 1\% retention while also ranking first in test performance.
On IMB2, however, the association weakens markedly ($R^2 = 0.10$, $p = 0.037$), with evenness explaining only 10\% of the variation in performance. GCcore is only mid-ranked in balance (evenness rank 6--11 of 15 at retention $\le 5\%$) while ranking first in test performance at 0.5 and 1\% retention. The most balanced coreset in fact performs below the median at both rates: Proto-Hard attains the highest evenness at 0.5\% retention (0.909, versus 0.730 for GCcore) yet ranks only 10 of 15 in test performance (0.389), and InfoMax-Hard attains the highest evenness at 1\% retention (0.920, versus 0.798 for GCcore) yet ranks 9 of 15 (0.441). A more even label distribution is therefore neither necessary nor sufficient for performance under class imbalance.

\begin{table}[t]
\centering
\caption{%
Mean rank of the 15 selection methods on each class-imbalanced training set; the subscript is the SD of that rank across five hyperparameter values. \textbf{Mean}: average across the three imbalanced training sets.}
\label{tab:imb_rank_sd}
\small
\begin{tabularx}{0.7\linewidth}{@{} l *{4}{>{\centering\arraybackslash}X} @{}}
\toprule
\multirow[c]{2}{*}{\makecell[l]{Selection\\ method}} & \multicolumn{3}{c}{Class-imbalanced training sets} & \multirow[c]{2}{*}{\textbf{Mean}} \\
\cmidrule(lr){2-4}
& IMB1 & IMB2 & IMB3 & \\
\midrule
GCcore       & $2.8$ & $4.8$ & $\underline{\mathbf{2.5}}$ & $\underline{\mathbf{3.3}}$ \\
K-means      & \ms{5.3}{2.9} & \ms{7.0}{4.4} & \ms{3.3}{2.1} & \ms{5.2}{3.1} \\
ELFS-Proto   & \msb{2.3}{1.5} & \ms{6.8}{2.0} & \ms{7.0}{2.0} & \ms{5.3}{1.8} \\
K-center     & \ms{4.8}{1.3} & \ms{5.0}{4.9} & \ms{6.3}{1.9} & \ms{5.3}{2.7} \\
ELFS-Pseudo  & \ms{6.3}{2.4} & \msb{4.0}{2.6} & \ms{7.8}{3.0} & \ms{6.0}{2.6} \\
InfoMax-Easy & \ms{8.8}{2.4} & \ms{6.5}{3.1} & \ms{3.0}{1.9} & \ms{6.1}{2.4} \\
Random       & \ms{4.5}{1.2} & \ms{9.3}{1.1} & \ms{5.8}{1.8} & \ms{6.5}{1.4} \\
Moderate     & \ms{5.3}{2.9} & \ms{6.3}{1.8} & \ms{8.5}{2.8} & \ms{6.7}{2.5} \\
CCS          & \ms{7.3}{2.7} & \ms{9.0}{3.2} & \ms{7.0}{1.0} & \ms{7.8}{2.3} \\
D\textsuperscript{2}-Proto     & \ms{8.3}{2.3} & \ms{5.8}{1.9} & \ms{11.3}{1.9} & \ms{8.4}{2.0} \\
InfoMax-Hard & \ms{11.8}{1.1} & \ms{7.0}{2.3} & \ms{9.0}{1.2} & \ms{9.3}{1.5} \\
ZCore        & \ms{13.5}{1.3} & \ms{11.0}{1.1} & \ms{10.5}{0.9} & \ms{11.7}{1.1} \\
D\textsuperscript{2}-Uniform   & \ms{12.3}{1.4} & \ms{12.8}{1.7} & \ms{11.8}{1.4} & \ms{12.3}{1.5} \\
Proto-Easy   & \ms{13.3}{0.3} & \ms{13.3}{2.7} & \ms{12.8}{1.1} & \ms{13.1}{1.4} \\
Proto-Hard   & \ms{14.0}{0.8} & \ms{11.8}{2.5} & \ms{13.8}{1.7} & \ms{13.2}{1.6} \\
\bottomrule
\end{tabularx}
\end{table}

\begin{figure}[t]
  \centering
  \includegraphics[width=\linewidth]{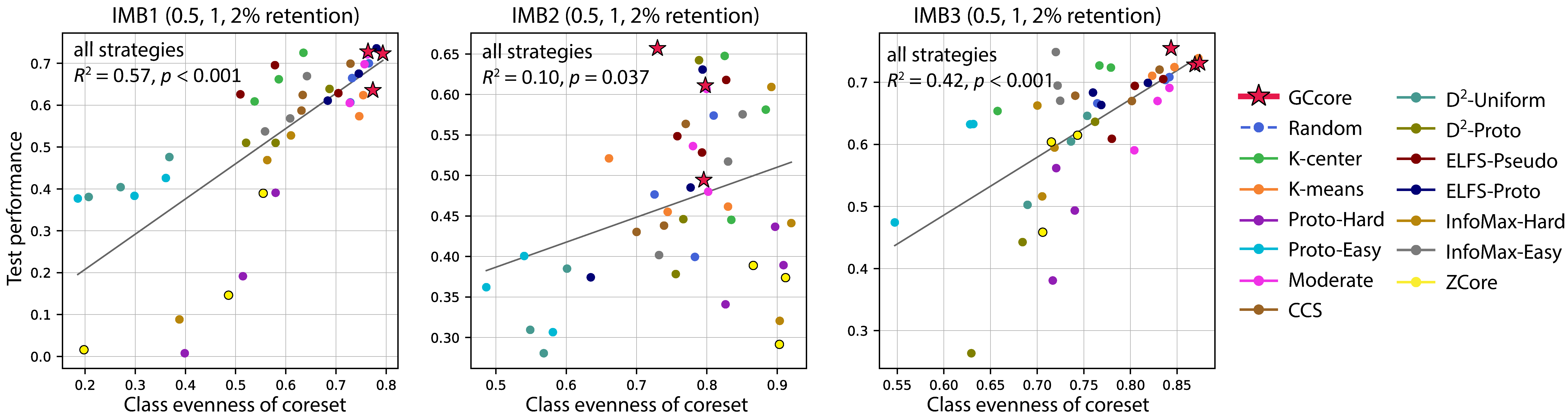}
  \caption{A class-balanced selection is neither necessary nor sufficient for better performance.}
  \label{fig:ablation_classimb_fig2}
\end{figure}

\subsection{GCcore is robust to the choice of distance metric}

GCcore is formulated for a general distance in Eq.~\ref{eq:di}. The derivation in Appendix~\ref{sec:guarantee} uses only two properties: that the pairwise dissimilarity is symmetric and nonnegative, and that $r_{\max}$ is the maximum pairwise value of that same dissimilarity over $\mathcal{S}$. Together, these give $\mathrm{sim}(x_i,x_j) \geq 0$ and $d_i(\mathcal{T}) \in [0, r_{\max}]$, which are all that Eq.~\ref{eq:equiv}, the identity $c_i(\mathcal{T}) = r_{\max} - d_i(\mathcal{T})$, and Lemma~\ref{lem:submod} require. Therefore, the $(1-1/e)$ guarantee of Eq.~\ref{eq:guarantee} holds for any distance metric. The selection procedure is otherwise unchanged: only the pairwise distances and $r_{\max}$ are recomputed.

We use the $\ell_1$ distance throughout the paper. To verify that the method is not tied to this choice, we repeat the selection with the Euclidean distance
\begin{equation}
    \lVert \mathbf{e}_i - \mathbf{e}_j \rVert_2 = \Big( \sum_{m=1}^{D} (e_{im} - e_{jm})^{2} \Big)^{1/2} ,
    \label{eq:l2}
\end{equation}
i.e., $d_i(\mathcal{T}) = \min_{x_j \in \mathcal{T}} \lVert \mathbf{e}_i - \mathbf{e}_j \rVert_2$
in Eq.~\ref{eq:di}, with $r_{\max}$ recomputed as the maximum pairwise $\ell_2$ distance over
$\mathcal{S}$.

Table~\ref{tab:rank_sd_l1l2} shows that both GCcore ($\ell_1$) and GCcore ($\ell_2$) rank ahead of every baseline on average, with mean ranks of 2.7 and 3.4 respectively, against 5.0 for the strongest baseline (Random). Per dataset, the two differ by a median of 1.0 rank positions, take the top two places on five datasets, and both fall within the top five on eight. GCcore's advantage over the baselines is therefore a property of its global coverage maximization mechanism, not of the distance used to instantiate it.

\begin{table}[t]
\centering
\caption{%
Mean rank of the 16 selection methods on each dataset, with GCcore's two distance variants treated as separate methods; the subscript is the SD of that rank across five hyperparameter values. GCcore is deterministic and reports no SD. \textbf{Mean}: average across the 10 datasets.}
\label{tab:rank_sd_l1l2}
\small
\setlength{\tabcolsep}{0pt}
\begin{tabular*}{\linewidth}{@{\extracolsep{\fill}} l ccc cccc ccc c @{}}
\toprule
\multirow[c]{3}{*}{\makecell[l]{Selection\\ method}}
& \multicolumn{3}{c}{WSI classification} & \multicolumn{4}{c}{Tile classification}
& \multicolumn{3}{c}{Segmentation} & \multirow[c]{3}{*}{\textbf{Mean}} \\
\cmidrule(lr){2-4} \cmidrule(lr){5-8} \cmidrule(lr){9-11}
& PANDA & \makecell{IMP-\\CRS} & \makecell{TCGA-\\BRCA} & \makecell{WSSS4\\LUAD}
& SICAP & \makecell{CRC-\\100k} & \makecell{Deep\\HP} & \makecell{BEE\\TLE}
& panCK & \makecell{CAME-\\LYON16} & \\
\midrule
GCcore ($\ell_1$) & $3.1$ & $\underline{\mathbf{1.5}}$ & $\underline{\mathbf{2.9}}$ & $\underline{\mathbf{1.5}}$ & $3.0$ & $4.0$ & $\underline{\mathbf{1.1}}$ & $\underline{\mathbf{1.8}}$ & $6.5$ & $\underline{\mathbf{1.5}}$ & $\underline{\mathbf{2.7}}$ \\
GCcore ($\ell_2$) & $4.6$ & $1.8$ & $6.4$ & $3.3$ & $4.0$ & $2.8$ & $1.9$ & $\underline{\mathbf{1.8}}$ & $5.5$ & $2.0$ & $3.4$ \\
Random        & \ms{7.5}{2.1} & \ms{5.8}{1.2} & \ms{3.0}{1.2} & \ms{6.3}{2.0} & \ms{3.8}{2.1} & \ms{5.5}{1.1} & \ms{4.5}{0.9} & \ms{3.0}{1.7} & \ms{4.3}{1.8} & \ms{6.5}{1.1} & \ms{5.0}{1.5} \\
K-means       & \msb{2.0}{1.0} & \ms{3.8}{1.5} & \ms{6.3}{2.9} & \ms{7.6}{1.9} & \ms{6.8}{3.5} & \msb{1.5}{1.1} & \ms{6.8}{0.7} & \ms{4.5}{0.9} & \ms{5.8}{2.3} & \ms{8.0}{1.8} & \ms{5.3}{1.7} \\
ELFS-Proto    & \ms{6.5}{1.5} & \ms{5.8}{1.5} & \ms{7.3}{3.3} & \ms{9.5}{1.5} & \msb{2.3}{3.4} & \ms{8.3}{1.1} & \ms{4.8}{0.5} & \ms{5.3}{1.2} & \ms{9.3}{1.8} & \ms{5.5}{0.7} & \ms{6.4}{1.6} \\
K-center      & \ms{3.5}{1.8} & \ms{7.3}{0.8} & \ms{12.0}{4.7} & \ms{7.3}{3.3} & \ms{5.3}{1.9} & \ms{4.8}{1.7} & \ms{9.3}{0.4} & \ms{7.5}{0.4} & \msb{3.8}{0.8} & \ms{3.8}{0.8} & \ms{6.4}{1.7} \\
Moderate      & \ms{8.8}{3.1} & \ms{10.0}{1.5} & \ms{5.0}{2.6} & \ms{10.1}{2.7} & \ms{6.5}{1.6} & \ms{5.3}{2.1} & \ms{5.0}{1.1} & \ms{4.8}{1.3} & \ms{9.0}{2.9} & \ms{8.3}{2.8} & \ms{7.3}{2.2} \\
ELFS-Pseudo   & \ms{4.8}{2.0} & \ms{7.0}{2.4} & \ms{11.3}{3.6} & \ms{4.8}{1.5} & \ms{9.3}{2.9} & \ms{9.0}{1.6} & \ms{8.0}{0.8} & \ms{7.8}{0.9} & \ms{9.8}{2.4} & \ms{5.0}{1.2} & \ms{7.7}{1.9} \\
CCS           & \ms{8.3}{1.1} & \ms{11.8}{2.7} & \ms{9.5}{2.8} & \ms{7.5}{1.3} & \ms{6.8}{1.8} & \ms{9.5}{2.4} & \ms{4.5}{0.2} & \ms{9.5}{3.4} & \ms{7.0}{2.3} & \ms{12.3}{3.9} & \ms{8.7}{2.2} \\
InfoMax-Easy  & \ms{9.8}{3.8} & \ms{11.3}{0.3} & \ms{11.0}{2.0} & \ms{8.3}{4.1} & \ms{11.8}{0.7} & \ms{7.5}{3.6} & \ms{9.5}{1.7} & \ms{9.5}{1.5} & \ms{7.3}{4.0} & \ms{7.8}{0.3} & \ms{9.4}{2.2} \\
ZCore         & \ms{12.5}{1.4} & \ms{9.8}{0.8} & \ms{12.5}{1.1} & \ms{10.5}{2.7} & \ms{11.3}{1.1} & \ms{8.0}{2.2} & \ms{12.0}{0.9} & \ms{11.3}{0.4} & \ms{6.0}{2.1} & \ms{7.0}{3.0} & \ms{10.1}{1.6} \\
InfoMax-Hard  & \ms{12.0}{1.5} & \ms{6.5}{1.5} & \ms{11.0}{1.8} & \ms{5.5}{1.4} & \ms{13.0}{1.7} & \ms{12.0}{2.0} & \ms{14.0}{0.4} & \ms{11.5}{0.5} & \ms{10.0}{1.5} & \ms{10.8}{2.3} & \ms{10.6}{1.5} \\
D\textsuperscript{2}-Proto      & \ms{10.0}{3.0} & \ms{10.0}{2.7} & \ms{6.8}{1.2} & \ms{9.5}{2.3} & \ms{13.0}{2.2} & \ms{13.8}{0.4} & \ms{14.5}{1.1} & \ms{14.5}{0.3} & \ms{7.0}{2.7} & \ms{12.8}{3.9} & \ms{11.2}{2.0} \\
D\textsuperscript{2}-Uniform    & \ms{12.0}{0.9} & \ms{15.3}{0.1} & \ms{4.5}{3.8} & \ms{13.8}{4.7} & \ms{8.5}{1.8} & \ms{15.8}{0.4} & \ms{12.8}{1.3} & \ms{15.3}{0.4} & \ms{15.0}{2.0} & \ms{15.9}{0.2} & \ms{12.9}{1.6} \\
Proto-Hard    & \ms{15.8}{1.6} & \ms{13.0}{1.4} & \ms{15.5}{2.7} & \ms{14.8}{2.1} & \ms{15.5}{3.4} & \ms{13.3}{0.4} & \ms{15.3}{0.5} & \ms{13.0}{0.5} & \ms{14.0}{2.0} & \ms{14.0}{1.5} & \ms{14.4}{1.6} \\
Proto-Easy    & \ms{15.0}{2.1} & \ms{15.8}{0.2} & \ms{11.3}{2.2} & \ms{16.0}{0.3} & \ms{15.5}{1.1} & \ms{15.3}{0.2} & \ms{12.3}{1.3} & \ms{15.3}{0.1} & \ms{16.0}{0.1} & \ms{15.1}{1.1} & \ms{14.7}{0.9} \\
\bottomrule
\end{tabular*}
\end{table}

\begin{table}[t]
\centering
\caption{%
Mean rank of the 15 selection methods on PANDA under the native slide-embedding dimensionality $D$ and the PCA-reduced dimensionality $D'$ (99\% explained variance; PRISM $1280\!\to\!189$, TITAN $768\!\to\!207$). The subscript is the SD of that rank across five hyperparameter values. GCcore is deterministic and reports no SD. \textbf{Mean}: average across the two slide encoders.}
\label{tab:rank_dim}
\small
\setlength{\tabcolsep}{3pt}
\begin{tabularx}{0.75\linewidth}{@{}l *{6}{Y}@{}}
\toprule
\multirow[c]{2}{*}{\makecell[l]{Selection\\ method}}
& \multicolumn{3}{c}{Full embedding ($\mathbf{e} \in \mathbb{R}^{D}$)}
& \multicolumn{3}{c}{Reduced embedding ($\mathbf{e} \in \mathbb{R}^{D'}$)} \\
\cmidrule(lr){2-4} \cmidrule(lr){5-7}
& PRISM & TITAN & \textbf{Mean} & PRISM & TITAN & \textbf{Mean} \\
\midrule
GCcore       & $3.0$ & $\underline{\mathbf{1.0}}$ & $\underline{\mathbf{2.0}}$ & $\underline{\mathbf{2.3}}$ & $\underline{\mathbf{2.8}}$ & $\underline{\mathbf{2.5}}$ \\
K-means      & \msb{2.0}{0.8} & \ms{6.8}{2.3} & \ms{4.4}{1.5} & \ms{3.5}{1.4} & \ms{4.8}{1.2} & \ms{4.1}{1.3} \\
ELFS-Proto   & \ms{5.8}{1.3} & \ms{3.8}{0.7} & \ms{4.8}{1.0} & \ms{5.8}{0.9} & \ms{3.3}{1.3} & \ms{4.5}{1.1} \\
InfoMax-Easy & \ms{8.8}{3.3} & \ms{8.5}{1.8} & \ms{8.6}{2.6} & \ms{2.5}{2.4} & \ms{6.5}{2.5} & \ms{4.5}{2.4} \\
ELFS-Pseudo  & \ms{4.5}{1.7} & \ms{5.3}{1.2} & \ms{4.9}{1.4} & \ms{3.0}{1.2} & \ms{7.0}{1.7} & \ms{5.0}{1.5} \\
Random       & \ms{6.5}{1.8} & \ms{3.8}{0.4} & \ms{5.1}{1.1} & \ms{7.0}{1.7} & \ms{3.3}{0.7} & \ms{5.1}{1.2} \\
CCS          & \ms{7.3}{0.9} & \ms{3.8}{3.0} & \ms{5.5}{2.0} & \ms{7.3}{0.7} & \ms{5.0}{1.4} & \ms{6.1}{1.0} \\
InfoMax-Hard & \ms{11.0}{1.4} & \ms{11.5}{1.1} & \ms{11.3}{1.3} & \ms{9.0}{1.3} & \ms{5.8}{0.9} & \ms{7.4}{1.1} \\
Moderate     & \ms{7.8}{2.8} & \ms{4.0}{1.8} & \ms{5.9}{2.3} & \ms{9.0}{2.9} & \ms{8.3}{1.8} & \ms{8.6}{2.4} \\
K-center     & \ms{3.0}{1.5} & \ms{10.5}{0.8} & \ms{6.8}{1.1} & \ms{7.8}{2.0} & \ms{10.0}{1.9} & \ms{8.9}{2.0} \\
D\textsuperscript{2}-Proto     & \ms{9.3}{2.7} & \ms{10.3}{2.5} & \ms{9.8}{2.6} & \ms{10.5}{1.0} & \ms{10.8}{2.2} & \ms{10.6}{1.6} \\
D\textsuperscript{2}-Uniform   & \ms{11.0}{0.9} & \ms{13.5}{2.9} & \ms{12.3}{1.9} & \ms{13.0}{1.8} & \ms{11.3}{1.9} & \ms{12.1}{1.9} \\
ZCore        & \ms{11.5}{1.3} & \ms{9.8}{0.8} & \ms{10.6}{1.0} & \ms{11.8}{0.5} & \ms{14.0}{0.2} & \ms{12.9}{0.4} \\
Proto-Hard   & \ms{14.8}{1.6} & \ms{13.5}{1.0} & \ms{14.1}{1.3} & \ms{12.8}{0.6} & \ms{13.3}{2.2} & \ms{13.0}{1.4} \\
Proto-Easy   & \ms{14.0}{2.1} & \ms{14.3}{1.6} & \ms{14.1}{1.8} & \ms{15.0}{0.9} & \ms{14.3}{0.6} & \ms{14.6}{0.7} \\
\bottomrule
\end{tabularx}
\end{table}

\subsection{GCcore is robust to the dimensionality of the FM embedding}
\label{ablation_dim}
As described in Section~\ref{exp_setup}, the WSI classification head maps a slide embedding $\mathbf{e} \in \mathbb{R}^{D}$ to the class logits, where $D$ is the native output dimensionality of the slide encoder ($D=1280$ for PRISM, $D=768$ for TITAN). To test the robustness of GCcore to the geometry of the embedding space, we applied PCA to the slide embeddings of PANDA and retained the leading components accounting for 99\% of the variance, giving $\mathbf{e} \in \mathbb{R}^{D'}$ with $D'=189$ for PRISM ($-85\%$) and $D'=207$ for TITAN ($-73\%$). 
We then repeated the comparison of all 15 selection methods in $\mathbb{R}^{D'}$ at low retention rates. Note that random selection does not use the embeddings, so its coresets are identical in both regimes.

Table~\ref{tab:rank_dim} shows that GCcore achieves the best mean rank in both regimes, 2.0 on the native embeddings and 2.5 on the reduced embeddings, and is the best-ranked method on the reduced embeddings of both FMs. More broadly, the ranking of the 15 selection strategies is largely preserved under dimensionality reduction: Spearman's rank correlation between the full and reduced embeddings is $\rho = 0.80$ for PRISM, $\rho = 0.82$ for TITAN and $\rho = 0.89$ for the mean over both encoders (all $p < 0.001$). In practice this means GCcore can be applied directly to compressed slide representations without modification, which is relevant given the large native dimensionalities across pathology FMs.

\end{document}